\documentclass[letterpaper]{article}
\usepackage[preprint]{aaai2027}

\usepackage[hyphens]{url}
\usepackage{graphicx}
\usepackage{natbib}
\usepackage{caption}
\usepackage{amsmath, amssymb, amsfonts, amsthm}
\usepackage{mathtools}
\usepackage{booktabs}
\usepackage{multirow}
\usepackage{microtype}

\newif\ifincludesupplement
\includesupplementtrue
\newif\ifsupplementbuild
\supplementbuildfalse

\newtheorem{assumption}{Assumption}
\newtheorem{proposition}{Proposition}
\newtheorem{lemma}{Lemma}
\newtheorem{corollary}{Corollary}

\DeclareMathOperator{\diag}{diag}

\DeclareMathOperator{\dist}{dist}
\newcommand{\R}{\mathbb{R}}
\newcommand{\E}{\mathbb{E}}
\newcommand{\D}{\mathcal{D}}
\newcommand{\Sbeta}{\mathcal{S}^A_{\beta_A}}
\newcommand{\beff}{b^{\mathrm{eff}}}
\newcommand{\Mdiag}{M(x)}
\newcommand{\bdry}{\partial G}
\newcommand{\Lhat}{\widehat{\mathcal{L}}}
\newcommand{\HH}{\mathcal{H}}

\ifsupplementbuild
\title{Supplementary Material: Should the Boundary Term Be Learned in Reflected Diffusion? Conormal Trace and Reflection Masking}
\else
\title{Should the Boundary Term Be Learned in Reflected Diffusion? Conormal Trace and Reflection Masking}
\fi

\author{
    Ziyue Wang\textsuperscript{\rm 1},
    Takafumi Kanamori\textsuperscript{\rm 1}
}
\affiliations{
    \textsuperscript{\rm 1}Institute of Science Tokyo
}

\begin{document}
\maketitle

\ifsupplementbuild
\else

\begin{abstract}
We study score learning for reflected diffusion on bounded domains. Reflection keeps trajectories feasible but does not ensure that the learned score satisfies the boundary behavior implied by the forward process. With implicit score matching, integration by parts leaves a boundary term, and we show that it depends on one scalar at each boundary point: the diffusion-weighted normal component of the score, or conormal trace. The no-flux condition fixes this value while leaving the remaining boundary components unrestricted; under anisotropic diffusion it generally differs from the ordinary normal score component. On hyperrectangles, our parametrization enforces the required trace without additional trainable parameters or a stochastic boundary estimator and, under regularity assumptions, can represent the true score, whereas fixing an incorrect value creates an error that more data cannot remove. We extend the construction to simplices and polygonal domains and identify reflection masking: hard reflection can keep samples feasible even when the learned trace is wrong, so post-reflection metrics may hide the error. Experiments show the clearest separation with less frequent reflection, anisotropic diffusion, and mass near intersections of constraints; under full reflection, final sample placement improves inconsistently, illustrating how hard repair can mask boundary-score errors and decouple score accuracy from downstream generation quality.
\end{abstract}

\section{Introduction}

Many generative tasks impose hard constraints: pixel values stay in a range, compositions remain on a simplex, and physical configurations stay inside a feasible set. Reflected diffusion enforces feasibility by reflecting a stochastic trajectory at the boundary \citep{lou2023reflecteddiffusionmodels, fishman2024diffusionmodelsconstraineddomains}. As in ordinary score-based diffusion, a neural network learns the score, the vector field used by the reverse-time dynamics to transform noise into data \citep{song2021scorebasedgenerativemodelingstochastic, ho2020denoisingdiffusionprobabilisticmodels}. Reflection keeps the numerical trajectory inside the domain, but does not ensure that this learned field obeys the boundary behavior of the forward process.

\begin{figure*}[t]
\centering
\includegraphics[width=\textwidth]{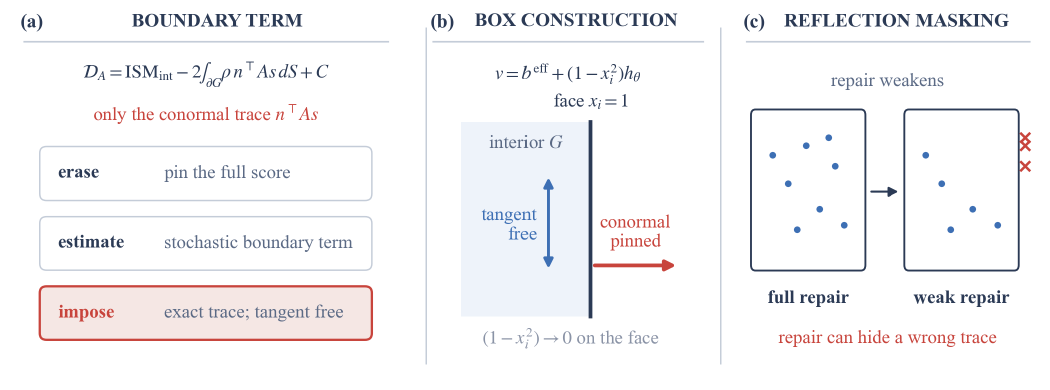}
\caption{The reflected-ISM boundary term sees only $n^\top As$, and hard reflection masks trace errors. (a) The boundary integral can be erased (zero-boundary), estimated (local-time), or imposed algebraically ($n^\top As=n^\top\beff$). (b) On a box, $v=\beff+(1-x_i^2)h$ pins the conormal component while leaving tangential components free. (c) Full repair folds infeasible proposals back into $G$; trace errors surface as repair weakens (M1, M4; Section~\ref{sec:experiments}).}
\label{fig:story}
\end{figure*}

This distinction matters for training. A standard objective, implicit score matching (ISM), removes the unknown data score by integration by parts \citep{hyvarinen2005estimation}. The boundary contribution is absent on an unbounded space, but remains as a surface term on a bounded domain. Ignoring or misspecifying it can change which score field the objective favors near the boundary.

The surface term depends on one scalar at each boundary point: the normal component of the diffusion-weighted score, which we call the conormal trace. The no-flux condition fixes this boundary value. In density estimation and generation, anisotropic diffusion can serve as a preconditioner for heterogeneous scales or strongly correlated targets, while spatially varying tensors can adapt to local geometry. The conormal trace reduces to the ordinary normal component under identity diffusion but generally differs in these broader settings.

Existing treatments make one of three choices: erase the term by pinning the score at the wall, estimate it with a stochastic local-time correction \citep{nordenhog2025scorebasedconstrainedgenerativemodeling}, or impose the value required by the forward process. We take the third route. On a hyperrectangle, a known boundary value plus a free neural field times a face-vanishing factor fixes only the required scalar and leaves the other boundary degrees of freedom trainable. Relative to the same backbone, this adds neither trainable parameters nor a stochastic boundary estimator, and makes the surface term independent of the network parameters during optimization. Under our assumptions the class still contains the true score. Because the boundary value is fixed by architecture, its residual is determined before training and cannot be corrected by optimization; this distinction drives both our theory and diagnostics. We extend the class face-wise to the simplex and evaluate a signed-distance version on polygonal domains.

Figure~\ref{fig:story} summarizes the boundary term, our parametrization, and reflection at evaluation. During sampling, an out-of-domain proposal is reflected back into the domain; this changes the point, not the learned score. Full reflection can therefore produce feasible samples despite a wrong boundary value. We call this \emph{reflection masking}. We separate score accuracy from sample behavior using the conormal-trace residual (CTR), no-reflection leakage, endpoint correction, maximum mean discrepancy (MMD), and sliced Wasserstein distance (SWD) under several repair schedules. Analytic boxes expose the architectural condition and score-error floor directly; simplex and polygon experiments provide scoped extensions. Full-repair placement remains condition-dependent, so the evidence supports score estimation rather than universal generation quality.

\paragraph{Contributions.} We identify the no-flux conormal trace seen by reflected ISM; give an exact hyperrectangle parametrization with representation and misspecification results, plus simplex and polygon extensions; and separate boundary-score accuracy from sample feasibility in theory and experiments.

\paragraph{Related work.} Truncated and generalized score-matching methods handle bounded domains by weighting the loss with functions that vanish at the boundary \citep{liu2022estimatingdensitymodelstruncation, yu2020generalizedscorematchinggeneral, yu2019generalizedscorematchingnonnegative, williams2024scorematchingtruncateddensity, scealy2020scorematchingcompositionaldistributions}. Our method instead uses the boundary value supplied by the no-flux condition of the reflected forward process. Work on reflected and constrained diffusion \citep{lou2023reflecteddiffusionmodels, fishman2024diffusionmodelsconstraineddomains, fishman2023metropolissamplingconstraineddiffusion}, reflected Schr\"odinger bridges and flow matching \citep{deng2024reflectedschrodingerbridgeconstrained, xie2024reflectedflowmatching}, mirror and projected diffusion \citep{liu2024mirrordiffusionmodelsconstrained, christopher2024constrainedsynthesisprojecteddiffusion}, and dual-constrained training \citep{khalafi2024constraineddiffusionmodelsdual} changes the stochastic path, geometry, sampler, or constraint-enforcement mechanism. We address a narrower question: how to train the score for a fixed reflected diffusion when ISM leaves a boundary term. Riemannian and simplex- or Dirichlet-based diffusion models use different geometries or forward processes \citep{debortoli2022riemannianscorebasedgenerativemodelling, huang2022riemanniandiffusionmodels, avdeyev2023dirichletdiffusionscoremodel, stark2024dirichletfm}. In particular, methods whose diffusion vanishes at the simplex boundary are outside the uniformly elliptic reflected setting considered here and are not direct baselines.

\section{Preliminaries and Problem Statement}
\label{sec:preliminaries}

This section introduces the forward diffusion, reflected probability flow, and estimation task.

\subsection{Score-based diffusion}
\label{sec:prelim-diffusion}

Let $X_0\sim p_{\mathrm{data}}$. The forward diffusion follows
\[
dX_t=b_t(X_t)\,dt+\sigma_t(X_t)\,dW_t,\qquad X_0\sim p_{\mathrm{data}},
\]
where $W_t$ is standard Brownian motion, $A_t(x)=\tfrac12\sigma_t(x)\sigma_t(x)^\top$ is the diffusion tensor, and $\rho_t$ is the density of $X_t$. We allow state-dependent $A_t$ and write $(\nabla\cdot A_t)_i=\sum_j\partial_jA_{t,ij}$. The score of the forward marginal is $s_t^\star(x)=\nabla_x\log\rho_t(x)$ and is approximated by a neural field $s_\theta(t,x)$ \citep{song2021scorebasedgenerativemodelingstochastic, ho2020denoisingdiffusionprobabilisticmodels}. The reverse-time drift depends on $s_t^\star$; in practice it uses $s_\theta$ and, for reflected diffusion, the boundary repair described below.

We sample $t\sim q$ on $[t_{\min},T]$, where $t_{\min}>0$ avoids possible nonsmoothness at $t=0$ and is the stopping time of the reverse solver. Denoising score matching regresses onto a conditional transition score. Gaussian transitions give a simple closed form \citep{song2020generativemodelingestimatinggradients}, while reflected diffusions can use constrained transition kernels or approximations \citep{lou2023reflecteddiffusionmodels}. ISM instead integrates by parts, so that only $s_\theta$ and its divergence appear \citep{hyvarinen2005estimation}:
\[
\E_{t,X_t}\bigl[\,\|s_\theta\|^2+2\,\nabla\!\cdot\!s_\theta\,\bigr]
\]
(up to a constant; the $A$-weighted version appears in Section~\ref{sec:theory}). We focus on ISM not because constrained denoising targets are impossible, but because it exposes the boundary integral directly and lets us study general diffusion tensors and domains without constructing a conditional transition score for each geometry.

\subsection{Reflected diffusion and the no-flux condition}
\label{sec:setup}

Let $G\subset\R^d$ be bounded, with closure $\bar G$, $C^2$ boundary $\bdry$, and outward unit normal $n$. Hyperrectangles are treated face-wise; edges and corners do not contribute to the $(d-1)$-dimensional surface integrals. A reflected diffusion adds a boundary local-time term that keeps $X_t$ in $\bar G$ \citep{lou2023reflecteddiffusionmodels, fishman2024diffusionmodelsconstraineddomains}. The local time increases only on $\bdry$. In the full-dimensional theory, $A_t$ is symmetric and uniformly elliptic, so its eigenvalues are bounded above and away from zero. The density satisfies $\partial_t\rho_t=-\nabla\cdot J_t$, and reflection imposes the no-flux condition:
\[
J_t=b_t\rho_t-A_t\nabla\rho_t-\rho_t\nabla\!\cdot\!A_t,\qquad J_t\cdot n=0\ \ \text{on }\bdry,
\]
where $J_t$ is the probability flux and $J_t\cdot n=0$ states that no net probability mass crosses the boundary; $\nabla\cdot A_t=0$ when $A_t$ is spatially constant. On the sampling side, the reverse SDE is run with \emph{repair}: a specified reflection map that returns an infeasible proposal to $G$ and changes the proposal, not the learned score. Under anisotropic $A_t$, continuous conormal reflection acts along $A_tn$. The coordinate-wise fold used in our simulations agrees with this direction when $A_tn\parallel n$ and is approximate otherwise; the discrepancy is included in the simulation (Sim) term of the excess-risk decomposition (Appendix~\ref{app:excess-risk}, Proposition~\ref{prop:excess}). This repair map is distinct from the Euclidean projection used only as an evaluation metric in Section~\ref{sec:diag}. We use classical pointwise boundary values; weak traces are deferred to Section~\ref{sec:disc-principle}.

\subsection{Problem statement}
\label{sec:problem}

We estimate the time-dependent score of a specified reflected forward process, rather than sampling from a given unnormalized energy. The inputs are samples $X_0\sim p_{\mathrm{data}}$ supported on $G$, the coefficients $b_t$ and $A_t$, and the geometry $(G,\bdry,n)$. The output is a field $s_\theta(t,x)$ that approximates $\nabla\log\rho_t$ in the interior and satisfies the boundary value implied by the same forward process. Training pairs $(t,X_t)$ are obtained by reflected noising, so $\rho_t$ need not be evaluated. Analytic densities used in controlled experiments are evaluation oracles, not training inputs. Sections~\ref{sec:diag}--\ref{sec:experiments} assess score estimation separately from reverse-time sample behavior. The central issue is that ISM on a bounded domain leaves a surface term; the next section derives the boundary value it requires.

\section{Flux-Compatible Score Classes}
\label{sec:theory}

We identify the boundary quantity in weighted ISM, derive its target from no flux, and construct score classes that satisfy it by design. We call such fields \emph{flux-compatible}.

\subsection{The weighted ISM boundary term}
\label{sec:ism-boundary}

For a candidate score field $s$, the $A$-weighted Fisher divergence is
\[
\D_A(s,s^\star)=\E_{t\sim q,\,X_t\sim\rho_t}\!\bigl[(s-s^\star)^\top A_t(s-s^\star)\bigr],
\]
with all fields evaluated at $(t,X_t)$. It reduces to the usual Fisher divergence when $A=I$. Expanding the square and integrating the cross term by parts gives
\[
\begin{aligned}
\D_A(s,s^\star)=\;&\underbrace{\E_{t,X_t}\!\bigl[s^\top A_t s+2\nabla\!\cdot\!(A_t s)\bigr]}_{\text{interior ISM objective }\mathcal L(s)}\\[2pt]
&-\;2\,\E_{t}\!\int_{\bdry}\rho_t\,n^\top\! A_t s\,dS\;+\;C,
\end{aligned}
\]
with $C$ independent of $s$. On an unbounded domain the surface integral is absent. On a bounded domain it need not vanish. For a vector field $w$, its normal trace is the scalar boundary function $n^\top w$. We therefore call $n^\top A_ts$ the \emph{conormal trace of the score}. The surface term depends on $s$ only through this value; when $A_t=I$, it is the ordinary normal trace.

\subsection{Standing assumptions}

\begin{assumption}
\label{ass:regularity}
For $t\in[t_{\min},T]$, $\rho_t\in C^1(\bar G)$, $\rho_t>0$ on $\bdry$, $A_t\in C^1(\bar G;\R^{d\times d})$ is symmetric uniformly elliptic, and all vector fields used below are regular enough for the divergence theorem. The no-flux condition $J_t\cdot n=0$ holds pointwise on $\bdry$.
\end{assumption}
These conditions justify pointwise integration by parts and division by $\rho_t$ on the boundary. On a hyperrectangle they apply on each face, while lower-dimensional intersections are ignored in surface measure. Singular data, vanishing boundary density, and weak reflected solutions require weak-trace variants.

\subsection{No-flux determines the conormal score trace}

Define the effective drift $\beff_t(x)=b_t(x)-\nabla\cdot A_t(x)$; when $A_t$ is spatially constant, $\beff_t=b_t$.

\begin{proposition}[No-flux conormal score trace]
\label{prop:conormal}
Under Assumption~\ref{ass:regularity},
\[
\boxed{\,n^\top A_t\nabla\log\rho_t=n^\top \beff_t\quad\text{on }\bdry.\,}
\]
For spatially constant $A_t$, this becomes $n^\top A_t\nabla\log\rho_t=n^\top b_t$; for $A_t=I$ it reduces to the normal-trace condition $\partial_n\log\rho_t=b_t\cdot n$.
\end{proposition}
\noindent(Proof in Appendix~\ref{app:proof-prop1}.)

Thus the forward process fixes one scalar linear condition, not the whole score vector. Under anisotropic or spatially varying $A_t$, the conormal trace can differ from the ordinary normal component $n^\top s$. Section~\ref{sec:exp-trace-law} makes this concrete: for $A=\diag(1,4)$, imposing the value appropriate to $A=I$ produces a fourfold error on the second-coordinate face.

\subsection{Boundary-free weighted ISM}

Let $\beta_A(t,x)=n(x)^\top \beff_t(x)$ and define the conormal-compatible class
\[
\Sbeta=\bigl\{s:n^\top A_t s(t,x)=\beta_A(t,x)\text{ on }\bdry\bigr\}.
\]
By Proposition~\ref{prop:conormal}, the true score $s^\star\in\Sbeta$, so restricting optimization to this class does not exclude the population minimizer.

\begin{proposition}[Conormal-compatible weighted ISM]
\label{prop:ism}
Let $s_\theta\in\Sbeta$. Under Assumption~\ref{ass:regularity},
\[
\boxed{\,\D_A(s_\theta,s^\star)=\E_{t,X_t}\!\bigl[s_\theta^\top A_t s_\theta+2\nabla\cdot(A_t s_\theta)\bigr]+C,\,}
\]
where $C$ is independent of $\theta$: on the trace-compatible class the boundary surface integral of Section~\ref{sec:ism-boundary} equals $\int_{\bdry}\rho_t\,\beta_A\,dS$ and absorbs into the constant. The interior weighted ISM objective and the weighted Fisher divergence therefore share minimizers over $\Sbeta$.
\end{proposition}
\noindent(Proof in Appendix~\ref{app:proof-prop2}.) The boundary integral does not have to be zero: on $\Sbeta$ it is independent of $\theta$ and is absorbed into $C$. No stochastic boundary estimate is required, although simulation, statistical, optimization, and approximation errors remain (Appendix~\ref{app:excess-risk}). In practice we minimize the empirical interior objective over reflected-forward samples and use a Hutchinson probe only for the Jacobian trace in $\nabla\cdot(A_ts_\theta)$.

\subsection{Exact hyperrectangle construction}
\label{sec:hyperrect-construction}

Let $G=[-1,1]^d$ and set $v_\theta(t,x)=A_t(x)s_\theta(t,x)$. On the relative interior of either face $x_i=\pm1$, flux compatibility is equivalent to $v_{\theta,i}=\beff_{t,i}$.

\begin{proposition}[Exact hyperrectangle construction]
\label{prop:hyperrect}
Let $h_\theta:[0,T]\times G\to\R^d$ be any smooth neural field and define
\[
\boxed{\begin{aligned}
v_{\theta,i}(t,x)&=\beff_i(t,x)+(1-x_i^2)h_{\theta,i}(t,x),\\
s_\theta&=A_t^{-1}v_\theta.
\end{aligned}}
\]
Then $n^\top A_t s_\theta=\beta_A$ on each face of $G$.
\end{proposition}
\noindent(Proof in Appendix~\ref{app:proof-prop3}.) Because $1-x_i^2=0$ on $x_i=\pm1$, the $i$th component of $v_\theta$ is fixed there, while its remaining components remain trainable. These are tangent to the face; when $A_t=I$, they are also the usual tangential score components. At a face intersection, one condition is imposed for each incident face. For reflected OU drift $b(t,x)=-\lambda(t)x$ and $A_t=I$, the construction becomes $s_{\theta,i}=-\lambda(t)x_i+(1-x_i^2)h_{\theta,i}$. It adds no trainable parameters relative to the common backbone, although $\beff_t$ and $A_t^{-1}$ must be evaluated. Lemma~\ref{lem:facewise} and Proposition~\ref{prop:nomisspec} show that every sufficiently regular compatible score has this form with a continuous $h^\star$; approximation by a finite neural network remains separate (Corollary~\ref{cor:approx}). The same backbone is shared across baselines, so differences are not caused by network capacity.

\subsection{Simplex and polygon extensions}
\label{sec:ctr-law}

\paragraph{Simplex extension.} The construction transfers face-wise to the probability simplex $\Delta^{K-1}$ under constant isotropic tangent diffusion $A=aP$, where $P=I-K^{-1}\mathbf1\mathbf1^\top$ projects onto the simplex tangent space. The mobility envelope $M_\Delta(x)=\diag(x)-xx^\top$ replaces $(1-x_i^2)$ as the face-vanishing factor: on $x_i=0$, its $i$th row is zero, so the network cannot change the required boundary component; in the relative interior, its range is the tangent space. Appendix~\ref{app:simplex-impl} gives the construction and representation result. Here $A=aP$ is uniformly elliptic on the tangent subspace, not in ambient $\R^K$; boundary-degenerate Wright--Fisher and Dirichlet-type diffusions use a different forward process and remain out of scope.

\paragraph{Polygon / signed-distance extension.} For a general polygon, coordinate-wise factors are unavailable. We instead use the signed distance, a nearest boundary point, a tangent projector, and a depth gate that vanishes on $\bdry$ (Appendix~\ref{app:smooth}). The construction fixes the conormal component locally while leaving the tangential component trainable. On nonconvex or disconnected domains, the nearest boundary point can be non-unique and the signed distance is nonsmooth on the medial axis. We therefore evaluate this extension empirically rather than claiming a global representation result (Section~\ref{sec:exp-sdf}, Experiment~4).

\paragraph{Boundary error determined by the parametrization.} For every structured class in our comparisons, the parametrization fixes $n^\top A_ts_\theta$ on the boundary before training. The flux-compatible class fixes it to $\beta_A$ and has zero conormal-trace residual (CTR); a class that pins another value retains a nonzero residual before and after optimization. Appendix~\ref{app:ctr-family} gives the closed forms, including the dependence on the drift and, for spatially varying $A$, on $\nabla\cdot A$. Proposition~\ref{prop:nomisspec} shows that the box construction contains every sufficiently regular compatible score. Under the additional assumptions of Proposition~\ref{prop:floor}, an incorrect fixed trace also yields a positive score-error floor for any amount of training data. A local-time estimator avoids fixed misspecification but introduces boundary-gradient variance, which reaches $900\times$ exact quadrature in our controlled estimator experiment. Appendix~\ref{app:statistical} states these results, and Sections~\ref{sec:exp-trace-law}--\ref{sec:exp-stat} test them.

\section{Diagnostics}
\label{sec:diag}

We evaluate two questions separately: whether the learned score satisfies the boundary condition of the forward process, and how samples behave under a specified repair schedule. Repair can enforce feasibility even when the score fails the first criterion. Let $m$ denote the reflection interval in the reverse solver: $m=1$ repairs after every numerical step, finite $m>1$ repairs every $m$ steps, and $m=\infty$ disables repair. Table~\ref{tab:diagnostics} summarizes the core diagnostics and their interpretation.

\begin{table*}[t]
\centering
\small
\begin{tabular}{lp{3.3cm}p{2.5cm}cl}
\toprule
diagnostic & definition & detects & masking-resist. & role \\
\midrule
CTR & $\E_t\!\int_{\bdry}|n^\top A_t s-\beta_A|^2 d\nu_t$ & wrong conormal trace on $\bdry$ & yes & primary (boundary) \\
leakNone & no-fold reverse-SDE exit fraction & boundary-drift error & yes & primary (operational) \\
ProjDisp & $\E\|X_{\mathrm{raw}}-\Pi_G(X_{\mathrm{raw}})\|$, raw endpoint & external repair burden & for $m>1$ & primary (operational) \\
face-mass error & $|$gen.\,$-$\,target$|$ near-face mass (active faces, in-domain) & near-face placement & no & supporting \\
boundary-band error & analogous near-$\bdry$ discrepancy (in-domain) & near-boundary placement & no & supporting \\
\midrule
$\mathrm{MMD}^2$, SWD & hard-reflected distributional distance & global distributional fit & no & secondary (masking contrast) \\
\bottomrule
\end{tabular}
\caption{Core diagnostics and their role in the evaluation hierarchy. CTR and leakNone remain informative under full repair; operational and distributional metrics gain trace sensitivity only as repair weakens (Section~\ref{sec:exp-lowK}).}
\label{tab:diagnostics}
\end{table*}

\paragraph{Boundary-value error.} The \textbf{conormal-trace residual} (CTR) $=\E_t\!\int_{\bdry}|n^\top A_ts-\beta_A|^2d\nu_t$ measures the squared error in the boundary quantity of Proposition~\ref{prop:conormal}. For polytopes, $\nu_t$ is uniform on each face and averaged across faces; Appendix~\ref{app:ctr-family} gives the closed forms. The density-weighted variant $\mathrm{CTR}_\rho$, with $d\nu_t=\rho_t\,dS$, controls the boundary-bias functional through Corollary~\ref{cor:ctrrho}. The proposed parametrization has CTR zero pointwise, independently of training. Unlike endpoint metrics, CTR evaluates the learned field itself and is unaffected by the repair schedule; we therefore treat it as the primary boundary diagnostic.

\paragraph{Behavior under reduced or no repair.} \textbf{No-reflection leakage} (leakNone) is the fraction of final reverse-SDE endpoints outside $\bar G$ in a separate run with $m=\infty$. Let $X_{\mathrm{raw}}$ be the endpoint before terminal projection and let $\Pi_G$ be Euclidean projection onto $\bar G$. \textbf{Projection displacement} (ProjDisp) $=\E\|X_{\mathrm{raw}}-\Pi_G(X_{\mathrm{raw}})\|$ measures the final endpoint correction, not cumulative reflection cost. Under full repair it is zero up to numerical tolerance by construction.

\paragraph{Placement and score error.} The \textbf{face-mass error} averages the absolute generated--target probability difference in prespecified bands near \emph{active faces}, i.e., constraints assigned nonnegligible target mass in those bands; target probabilities are analytic or estimated with an independent high-accuracy reference sampler. The \textbf{boundary-band error} is the analogous quantity near the full boundary. Both use generated samples in $\bar G$ and measure repaired-sample placement. When an analytic score is available, $\mathrm{ShellErr}_{A,\delta}=\E[(s-s^\star)^\top A_t(s-s^\star)\mid\dist(X_t,\bdry)\le\delta]$ measures weighted score error in a boundary layer of width $\delta$. Tangential boundary error and the simplex unit-normal convention are defined in Appendix~\ref{app:ctr-family}. We also report $\mathrm{MMD}^2$ and sliced 1-Wasserstein distance (SWD) for the repaired distribution, but do not treat them as direct tests of boundary-trace correctness under full repair.

\begin{figure}[t]
\centering
\includegraphics[width=\linewidth]{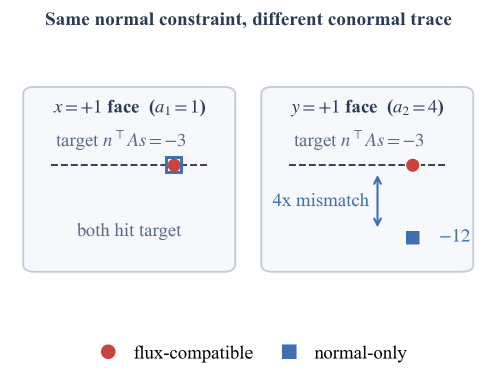}
\caption{Under anisotropic diffusion, the no-flux condition fixes the conormal rather than the normal trace. For $A=\operatorname{diag}(1,4)$ and $b=-3x$, the target $n^\top As=n^\top b=-3$ is identical on the positive $x$- and $y$-faces. Enforcing $n^\top s=-3$ is correct on $x=+1$, where $a_1=1$, but yields $n^\top As=-12$ on $y=+1$, where $a_2=4$. The flux-compatible class matches the target on both faces. Values are analytic and have no seed variance.}
\label{fig:a2-trace}
\end{figure}

\begin{figure*}[t]
\centering
\includegraphics[width=\textwidth]{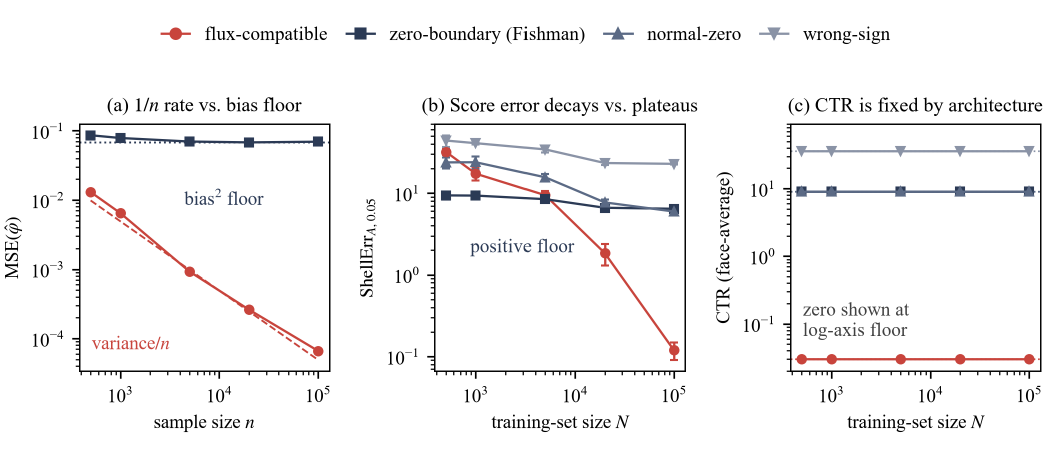}
\caption{Statistical consequences of boundary parametrization. (a) In the identifiable family $s_\phi=b+\phi\nabla c$, $c(x,y)=\cos(\pi x)\cos(\pi y)$, $\phi_0=0.5$; the flux-compatible estimator follows the $1/n$ MSE guide, while the Fishman-style family approaches the squared bias of its pseudo-true value $\phi_F^\star=0.760$. (b) Boundary-shell $A$-weighted score error decreases with training-set size $N$ for flux-compatible scores but plateaus for the three incorrect-trace classes; all use the same Jacobian stabilizer. Error bars show mean $\pm$ one standard deviation over five paired seeds. (c) Face-average CTR stays fixed at the architecture-imposed values $0,9,9,36$ for flux-compatible, zero-boundary (Fishman), normal-zero, and wrong-sign. The zero value is drawn at the log-axis floor; dotted lines are closed-form values.}
\label{fig:stat-triptych}
\end{figure*}

\begin{figure*}[t]
\centering
\includegraphics[width=\textwidth]{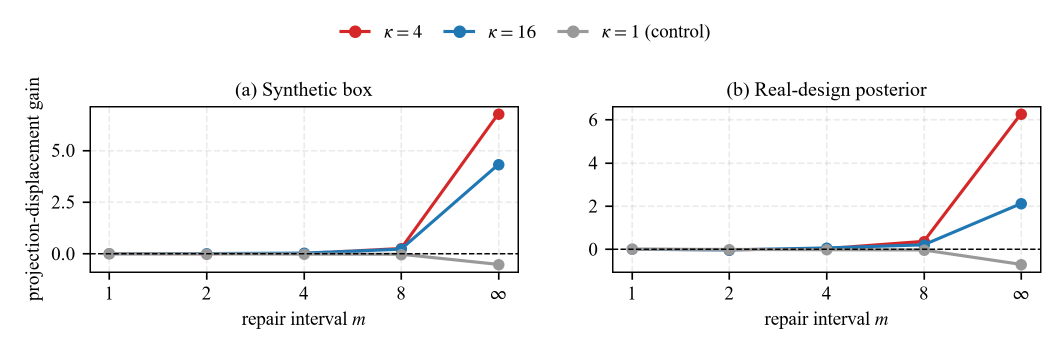}
\caption{Projection-displacement gain, $\min_{b\in\mathcal B_{\mathrm{nonflux}}}\mathrm{ProjDisp}_b-\mathrm{ProjDisp}_{\mathrm{flux}}$, versus repair interval $m$ for a corner-active GMM with two active faces (left) and a three-dimensional box-truncated Gaussian with diabetes-design covariance (right). Positive values favor flux-compatible scores. Here $A=Q\diag(1,\kappa)Q^\top$; $\kappa=1$ is isotropic and $Q$ is non-axis-aligned in rotated cells. The gain is masked at $m=1$ and emerges toward $m=\infty$; face-mass placement is in Appendices~\ref{app:phase-diagram-full} and~\ref{app:real-design-full}.}
\label{fig:repair-spectrum}
\end{figure*}

Experiment~3 also reports an \textbf{active-sign error}, the signed generated--target near-face mass difference on the intended side of each active face. It is distinct from ProjDisp; correlation details are in Appendix~\ref{app:real-design-full}.

\section{Experiments and Results}
\label{sec:experiments}

We use the experiments to answer three questions: whether the conormal trace is the boundary quantity that matters, whether a wrong fixed trace creates a statistical floor, and when that fixed error becomes visible in sampler-level metrics. Sections~\ref{sec:exp-trace-law}--\ref{sec:exp-stat} form Experiment~1; Experiments~2--3 (Section~\ref{sec:exp-lowK}) study reflection masking on synthetic and real-design boxes; Experiment~4 (Section~\ref{sec:exp-sdf}) is a scoped signed-distance test on a disconnected polygon. Simplex/HITChip diagnostics remain in Appendix~\ref{app:e5b-full}.

\paragraph{Setup, baselines, and gains.} Methods share data, backbone, optimizer, sampler, repair schedule, and seeds within each condition. The structured baselines differ only in their boundary composition: flux-compatible pins $n^\top As=\beta_A$; Fishman-style zero-boundary pins the whole score to zero; normal-zero pins only the Euclidean normal component to zero; wrong-sign flips the required conormal trace; wrong-normal-only imposes the trace computed as if $A=I$. For any smaller-is-better metric, define $\Delta_{\rm metric}=\min_{b\in\mathcal B_{\rm nonflux}}{\rm metric}_b-{\rm metric}_{\rm flux}$, so positive ProjDisp or face-mass gain favors flux-compatible scores. Experiments~2--3 use $A=Q\diag(1,\kappa)Q^\top$, with eigenvalue ratio $\kappa$ and non-axis-aligned $Q$ giving rotated anisotropy. Experiment~2 is a two-dimensional corner-active GMM with mass near two intersecting faces. Experiment~3 is a three-dimensional real-design posterior: a box-truncated Gaussian whose covariance and rotation come from the diabetes design matrix and whose mass lies near three faces.

\paragraph{Reflection-masking hypotheses.} M1: under full repair ($m=1$), post-repair metrics can hide nonzero CTR; ProjDisp is zero up to tolerance by construction and face-mass gains can have either sign. M2: in the $A=I$ control, flux-compatible and wrong-normal-only impose the same trace. M3: in the single-active-face, full-repair control, no operational gain is expected. M4: as repair weakens, fixed trace misspecification should appear in operational metrics, especially under rotated anisotropy and multiple active faces. Experiments~2--3 test M1, M2, and M4; Appendix~\ref{app:neg-controls} tests M2--M3. Supplementary Table~\ref{tab:scope} gives the full prediction--outcome scorecard. MMD and SWD remain secondary because hard repair can mask score errors; baseline formulas and estimator controls are in Appendices~\ref{app:baseline-formulas} and~\ref{app:e3-full}.

\subsection{Experiment 1a --- Conormal, not normal, trace determines the boundary error}
\label{sec:exp-trace-law}

On $[-1,1]^2$ with drift $b=-\lambda x$, we use the analytic density $\log\rho_{\lambda,\epsilon,A}(x,y)=-\frac{\lambda}{2}(x,y)^\top A^{-1}(x,y)+\epsilon\cos(\pi x)\cos(\pi y)-\log Z_A$, whose oscillatory derivative vanishes on each face, so the no-flux conormal target is $\beta_A=n^\top b$ (Appendix~\ref{app:a2-full}). At $A=\diag(1,4)$ and $(\lambda,\epsilon)=(3,0.5)$, enforcing the normal trace as if $A=I$ is fourfold off target on the eigenvalue-four face, whereas the flux-compatible construction is exact (Figure~\ref{fig:a2-trace}). The measured CTR equals its closed form with zero seed variance. Flux-compatible scores attain $A$-weighted score and shell errors $2.79\times10^{-3}$ and $3.09\times10^{-3}$; the best values across the incorrect-trace classes are $2.13\times10^{-1}$ and $3.959$, respectively (Supplementary Table~\ref{tab:a2-full}).

The $A=I$ controls, score-error fields, and tangential-trace fields are in Appendix~\ref{app:e1-full}. Across the state-dependent-diffusion grid, measured CTR equals its closed-form prediction to floating-point precision and is unchanged by training (Appendix~\ref{app:e6-training}); Appendix~\ref{app:e3-full} gives the boundary-estimator comparison.

\paragraph{Paired Fishman-style comparison.} Appendix~\ref{app:baseline-formulas} compares only the Fishman-style zero-boundary parametrization under the same reflected-ISM pipeline, not the full system of \citet{fishman2024diffusionmodelsconstraineddomains}. At $b=0$, both flux-compatible and zero-boundary classes impose the correct zero conormal trace, but zero-boundary also erases the tangential boundary score. When $b\ne0$, zero-boundary additionally imposes the wrong conormal value; anisotropic $A$ can further separate normal-only from conormal constraints. The operational gap appears once repair is removed.

\subsection{Experiment 1b --- Sample-size scaling and the misspecification floor}
\label{sec:exp-stat}

Figure~\ref{fig:stat-triptych} checks the predictions of Appendix~\ref{app:statistical}. Panel~(a) uses $s_\phi=b+\phi\nabla c$, $c(x,y)=\cos(\pi x)\cos(\pi y)$, with truth $\phi_0=0.5$. The flux-compatible family contains the truth; the Fishman-style family does not and converges to the pseudo-true value $\phi_F^\star=0.760$. Panels~(b)--(c) sweep the training-bank size $N$ in the same box setting. All four classes use the same non-capping Jacobian stabilizer $\lambda_{\rm Jac}\E\|\nabla s\|_F^2$ with $\lambda_{\rm Jac}=10^{-2}$ to prevent finite-bank divergence wells (Appendix~\ref{app:n-sweep}). The boundary-shell risk drops from $32.0$ to $0.12$ for flux-compatible scores but plateaus at $6.4$, $6.0$, and $22.8$ for zero-boundary, normal-zero, and wrong-sign. CTR remains exactly $0/9/9/36$ with zero seed variance.

\subsection{Experiments 2--3 --- When fixed trace errors become operationally visible}
\label{sec:exp-lowK}
\label{sec:exp-realdesign}

CTR directly measures trace misspecification and is unaffected by repair. Experiments~2--3 ask when this fixed structural error becomes visible in sampler-level diagnostics. Both experiments use $\kappa\in\{1,4,16\}$, repair intervals $m\in\{1,2,4,8,\infty\}$, and six architecture-matched baselines. Experiment~2 uses the two-face GMM; Experiment~3 uses the three-dimensional real-design posterior with a Gibbs/TMVG reference sampler. Full target definitions, grids, controls, and seed counts are in Appendices~\ref{app:phase-diagram-full}, \ref{app:real-design-full}, and~\ref{app:neg-controls}.

The fold repairs sample locations rather than the learned score. Thus classes with nonzero CTR can look similar under full repair, while the same misspecification appears through leakage and projection displacement as repair weakens. Figure~\ref{fig:repair-spectrum} shows this transition. At rotated $\kappa=16$, ProjDisp gain increases from approximately zero at $m=1$ to $4.32$ on the synthetic target and $2.11$ on the real-design target at $m=\infty$. The $A=I$ control shows no flux-vs-wrong-normal distinction (M2), and the separate single-active-face controls are null (M3). We do not claim a one-to-one link between CTR and generation quality: full-repair face-mass placement is condition-dependent, and secondary datasets and extreme-drift behavior are reported in Appendices~\ref{app:e4-full}, \ref{app:flower}, \ref{app:extreme}, and~\ref{app:scope}.

\subsection{Experiment 4 --- Signed-distance extension on a disconnected polygon}
\label{sec:exp-sdf}

Appendix~\ref{app:sdf-full} evaluates the signed-distance-function construction on a disconnected polygon made of a nonconvex L-polygon, a square island, and a narrow rectangle. This extension lacks the box no-misspecification theorem, so it is an empirical scope test; all methods share the same Jacobian stabilizer with $\lambda_{\rm Jac}=10^{-2}$. Against an equal-CTR scalar mask at $m=\infty$, the tangent-preserving construction reduces ProjDisp from $0.407$ to $0.224$ and boundary-band error from $0.226$ to $0.183$, while component KL worsens from $0.013$ to $0.209$ (Supplementary Table~\ref{tab:sdf-qt-final}). Thus the extension improves boundary correction in this test but not component allocation or general generation quality.

\section{Discussion and Scope}
\label{sec:disc-principle}

\paragraph{Theory.} Flux-compatible ISM separates boundary-trace correctness from interior expressivity. Proposition~\ref{prop:conormal} supplies the conormal trace from the forward process, and Proposition~\ref{prop:hyperrect} gives a box parametrization that meets it while leaving the backbone, hence the tangential score, free. Section~\ref{sec:ctr-law} extends this logic face-wise to simplices and by a signed-distance construction. CTR is architectural, not learned, as Section~\ref{sec:exp-trace-law} verifies.

\paragraph{Experimental interpretation.} CTR exposes trace misspecification that full repair can hide in MMD/SWD and placement metrics. No-reflection leakage is informative when trajectories reach the boundary but can lose power otherwise. The largest boundary-band, face-mass, and ProjDisp gains occur with rotated anisotropy, multiple active faces, and weak repair. Full-repair placement is condition-dependent; we therefore claim boundary correctness, not a universal generation-quality gain. It matters most when invalid outputs are unacceptable or projections between repairs are costly.

\paragraph{Limitations.} The no-misspecification theorem is exact only for hyperrectangles under Assumption~\ref{ass:regularity}; weak-trace variants remain open. The simplex result is face-wise, and the polygon test is empirical with a shared Jacobian penalty. Moderate-dimensional simplex/HITChip experiments are structural diagnostics; empirical ISM can overfit the divergence. Near-delta and single-face controls are inconclusive or null (Supplementary Table~\ref{tab:scope}). In the box setting, conormal compatibility adds no trainable parameters or stochastic boundary estimator; its operational benefit is scoped to anisotropy, multiple active constraints, and weak repair.

\paragraph{Generative AI disclosure.} Generative AI tools were used for language refinement, manuscript review, and code/package auditing. All scientific claims, proofs, experiments, code, and submitted artifacts were independently verified and approved by the authors.

\fi

\ifincludesupplement
\ifsupplementbuild
\else
  \clearpage
  \section*{Supplementary Material}
\fi
\appendix
\setcounter{figure}{0}
\setcounter{table}{0}
\setcounter{assumption}{0}
\setcounter{proposition}{0}
\setcounter{lemma}{0}
\setcounter{corollary}{0}
\renewcommand{\thefigure}{S\arabic{figure}}
\renewcommand{\thetable}{S\arabic{table}}
\renewcommand{\theassumption}{S\arabic{assumption}}
\renewcommand{\theproposition}{S\arabic{proposition}}
\renewcommand{\thelemma}{S\arabic{lemma}}
\renewcommand{\thecorollary}{S\arabic{corollary}}

\section{Proofs}

\subsection{Proof of Proposition~\ref{prop:conormal}}
\label{app:proof-prop1}
The no-flux condition gives $0=J_t\cdot n=\rho_t b_t\cdot n-n^\top A_t\nabla\rho_t-\rho_t n^\top(\nabla\cdot A_t)$. Since $\nabla\rho_t=\rho_t\nabla\log\rho_t$ and $\rho_t>0$ on $\bdry$, division by $\rho_t$ gives $n^\top A_t\nabla\log\rho_t=n^\top(b_t-\nabla\cdot A_t)=n^\top \beff_t$. If $A_t$ is spatially constant, $\nabla\cdot A_t=0$; if $A_t=I$, the identity becomes $\partial_n\log\rho_t=b_t\cdot n$. \qed

\subsection{Proof of Proposition~\ref{prop:ism}}
\label{app:proof-prop2}
Expand $\D_A(s,s^\star)=\E[s^\top A s]-2\E[s^\top A s^\star]+\E[(s^\star)^\top A s^\star]$. For fixed $t$, the cross term is $\E_{\rho_t}[s^\top A s^\star]=\int_G(As)\cdot\nabla\rho_t\,dx$. By the divergence theorem,
\[
\int_G(As)\cdot\nabla\rho_t\,dx=\int_{\bdry}\rho_t\,n^\top As\,dS-\int_G\rho_t\,\nabla\cdot(As)\,dx.
\]
If $s\in\Sbeta$, then $n^\top As=\beta_A$ on $\bdry$, so the boundary integral equals $\int_{\bdry}\rho_t\,\beta_A\,dS$, independent of trainable parameters. Absorbing this term and the true-score norm $\E[(s^\star)^\top A s^\star]$ into a constant $C$ and averaging over $t\sim q$ gives $\D_A(s,s^\star)=\E_{t,X_t}[s^\top As+2\nabla\cdot(As)]+C$. \qed

\subsection{Proof of Proposition~\ref{prop:hyperrect}}
\label{app:proof-prop3}
On the face $x_i=1$, $n=e_i$ and $n^\top As=v_i$; on $x_i=-1$, $n=-e_i$ and $n^\top As=-v_i$. The target trace $n^\top \beff$ equals $\beff_i$ on $x_i=1$ and $-\beff_i$ on $x_i=-1$. The construction $v_{\theta,i}=\beff_i+(1-x_i^2)h_{\theta,i}$ satisfies $1-x_i^2=0$ on both faces, giving $v_i=\beff_i$ on both, which matches both target identities. \qed

\subsection{Face-wise factorization}
\label{app:proof-lem1}
\begin{lemma}[Face-wise factorization]
\label{lem:facewise}
Let $u_i\in C^1([-1,1]^d)$ satisfy $u_i(x)=0$ whenever $x_i=\pm1$. Then there exists a continuous function $h_i$ such that $u_i(x)=(1-x_i^2)h_i(x)$; for $|x_i|<1$, $h_i=u_i/(1-x_i^2)$, with continuous face extensions $h_i(\pm 1,x_{-i})=\mp\tfrac12\partial_i u_i(\pm 1,x_{-i})$.
\end{lemma}
\noindent\emph{Proof.} For $|x_i|<1$, set $h_i(x)=u_i(x)/(1-x_i^2)$. Since $u_i(1,x_{-i})=0$ and $u_i\in C^1$, Taylor expansion in the $i$-th coordinate gives $u_i(x_i,x_{-i})=(x_i-1)\partial_i u_i(1,x_{-i})+o(x_i-1)$ as $x_i\to 1$. Combined with $1-x_i^2=-(x_i-1)(1+x_i)$, this gives $\lim_{x_i\to 1}u_i(x)/(1-x_i^2)=-\tfrac12\partial_i u_i(1,x_{-i})$. The $x_i\to -1$ limit is analogous and equals $\tfrac12\partial_i u_i(-1,x_{-i})$. Continuity follows from $u_i\in C^1$. \qed

\subsection{No boundary misspecification on boxes}
\label{app:proof-prop4}
\begin{proposition}[No boundary misspecification on boxes]
\label{prop:nomisspec}
Assume $G=[-1,1]^d$, and for each fixed $t$ let $s^\star\in C^1(\bar G;\R^d)$, $A\in C^1(\bar G)$ uniformly elliptic, and $\beff\in C^1(\bar G;\R^d)$. If $s^\star$ satisfies the no-flux conormal trace, then there exists a continuous $h^\star(t,\cdot):\bar G\to\R^d$ such that
\[
\boxed{\begin{aligned}
s^\star(t,x)&=A^{-1}\!\bigl[\beff+\Mdiag h^\star(t,x)\bigr],\\
\Mdiag&=\diag(1-x_1^2,\ldots,1-x_d^2).
\end{aligned}}
\]
\end{proposition}
\noindent\emph{Proof.} For each fixed $t$, define $v^\star=As^\star$ and $u_i=v_i^\star-\beff_i$. On each face $x_i=\pm 1$, the conormal trace condition gives $v_i^\star=\beff_i$, hence $u_i=0$. Lemma~\ref{lem:facewise} yields $u_i=(1-x_i^2)h_i^\star(t,\cdot)$ for continuous $h_i^\star$. Stacking coordinates, $v^\star=\beff+\Mdiag h^\star$, and since $A$ is uniformly elliptic, $s^\star=A^{-1}[\beff+\Mdiag h^\star]$. The function $h^\star$ inherits the $t$-dependence of $s^\star$. \qed

\subsection{Approximation transfer (Corollary~\ref{cor:approx})}
\label{app:proof-cor1}

\begin{corollary}[Approximation transfer]
\label{cor:approx}
If $A$ is uniformly elliptic with $\lambda_{\min}(A)\ge a_0>0$ and a neural class approximates $h^\star$ with $\E\|h_\theta-h^\star\|^2\le\varepsilon^2$, then the induced score satisfies $\D_A(s_\theta,s^\star)\le a_0^{-1}\varepsilon^2$ up to the vector-norm convention.
\end{corollary}

\emph{Proof.} Let $s_\theta=A^{-1}[\beff+Mh_\theta]$ and $s^\star=A^{-1}[\beff+Mh^\star]$. Then $s_\theta-s^\star=A^{-1}M(h_\theta-h^\star)$, so $(s_\theta-s^\star)^\top A(s_\theta-s^\star)=(h_\theta-h^\star)^\top MA^{-1}M(h_\theta-h^\star)$. Since $\lambda_{\min}(A)\ge a_0$, $\|A^{-1}\|\le a_0^{-1}$ and $\|M\|\le 1$, so $\D_A(s_\theta,s^\star)\le a_0^{-1}\,\E\|h_\theta-h^\star\|^2\le a_0^{-1}\varepsilon^2$. Constants depending on the vector-norm convention are suppressed. \qed

\subsection{Boundary bias and excess-risk decomposition (Proposition~\ref{prop:excess})}
\label{app:excess-risk}
\label{app:proof-prop5}

For a general score field $s$, define the interior weighted ISM functional $\mathcal L(s)=\E_P[s^\top As+2\nabla\cdot(As)]$ and the boundary-bias functional
\[
\mathcal B(s)=\E_t\!\int_{\bdry}\rho_t(x)\!\bigl[n^\top A_t s(t,x)-\beta_A(t,x)\bigr]dS.
\]
The integration-by-parts identity gives $\D_A(s,s^\star)=\mathcal L(s)+C-2\mathcal B(s)$, and $\mathcal B(s)=0$ when $s\in\Sbeta$. Let $P_\Delta$ be the distribution generated by a discretized reflected forward simulator and define $\mathcal L_\Delta$ accordingly; let $\Lhat_{N,\Delta}$ be its empirical version. For $\HH\subset\Sbeta$, let $s_\HH$ minimize $\D_A$ over $\HH$, and let $\hat s$ satisfy $\Lhat_{N,\Delta}(\hat s)\le\inf_{s\in\HH}\Lhat_{N,\Delta}(s)+\varepsilon_{\mathrm{opt}}$.

\begin{proposition}[Excess-risk decomposition]
\label{prop:excess}
For $\HH\subset\Sbeta$,
\[
\boxed{\;
\begin{aligned}
\D_A(\hat s,s^\star)\le\;&\underbrace{\D_A(s_\HH,s^\star)}_{\mathrm{Approx}}
+2\underbrace{\sup_{s\in\HH}|\Lhat_{N,\Delta}(s)-\mathcal L_\Delta(s)|}_{\mathrm{Stat}}\\
&+2\underbrace{\sup_{s\in\HH}|\mathcal L_\Delta(s)-\mathcal L(s)|}_{\mathrm{Sim}}+\varepsilon_{\mathrm{opt}}.
\end{aligned}
\;}
\]
For a non-trace-compatible class, the same argument incurs an additional boundary-bias term $2\sup_{s,s'\in\HH}|\mathcal B(s)-\mathcal B(s')|\le 4\sup_{s\in\HH}|\mathcal B(s)|$.
\end{proposition}

\emph{Proof.} For $s\in\HH\subset\Sbeta$, Proposition~\ref{prop:ism} gives $\D_A(s,s^\star)=\mathcal L(s)+C$, so $\D_A(\hat s,s^\star)-\D_A(s_\HH,s^\star)=\mathcal L(\hat s)-\mathcal L(s_\HH)$. Decomposing this difference by adding and subtracting $\mathcal L_\Delta$ and $\Lhat_{N,\Delta}$ at both $\hat s$ and $s_\HH$, the empirical optimization condition bounds the middle empirical difference by $\varepsilon_{\mathrm{opt}}$, the two empirical-process terms by $2\sup_{s\in\HH}|\Lhat_{N,\Delta}(s)-\mathcal L_\Delta(s)|$, and the two simulation terms by $2\sup_{s\in\HH}|\mathcal L_\Delta(s)-\mathcal L(s)|$. Adding $\D_A(s_\HH,s^\star)$ gives the stated bound. For non-trace-compatible $\HH$, the identity becomes $\D_A(s,s^\star)=\mathcal L(s)+C-2\mathcal B(s)$; comparing two elements introduces $-2[\mathcal B(s)-\mathcal B(s')]$, bounded by $2\sup_{s,s'\in\HH}|\mathcal B(s)-\mathcal B(s')|\le 4\sup_{s\in\HH}|\mathcal B(s)|$. \qed

\bigskip
This decomposition is deterministic: it separates approximation, statistical, simulation, and optimization errors without assuming finite-sample rates. Flux-compatible ISM removes boundary misspecification and the trainable boundary-estimator term; it does not eliminate the other three.

\subsection{Statistical role of the boundary parametrization}
\label{app:statistical}
\label{app:floor}

This appendix expands the statistical reading of Section~\ref{sec:theory}. The boundary parametrization is not a sampling heuristic; it changes the estimation problem itself. For a general score field $s$, define the boundary-bias functional $\mathcal B(s)=\E_t\!\int_{\bdry}\rho_t\,[n^\top A_t s-\beta_A]\,dS$, so that the integration by parts of Section~\ref{sec:ism-boundary} reads $\D_A(s,s^\star)=\mathcal L(s)+C-2\mathcal B(s)$ with $C$ independent of $s$: interior ISM estimates the score risk up to $\mathcal B$, and the choice of class decides what happens to $\mathcal B$.

The three treatments of Figure~\ref{fig:story} then separate statistically. \emph{Pinned} classes (zero-boundary, normal-zero, wrong-sign, wrong-normal-only; Appendix~\ref{app:baseline-formulas}) fix $n^\top A_ts$ on $\bdry$ to a $\theta$-independent function, so $\mathcal B$ is constant on the class and interior ISM shares its minimizer with $\D_A$ --- a \emph{faithful} within-class criterion. Their cost is \emph{representational bias}: unless the pinned trace is the no-flux trace, the class excludes $s^\star$ and a risk floor remains (Proposition~\ref{prop:floor}). An \emph{estimated} boundary term (local-time) has no representational bias but re-introduces the boundary integral as a trainable stochastic term, whose gradient variance exceeds exact quadrature by $\approx900\times$ at $\lambda=1$ (Appendix~\ref{app:e3-full}). The \emph{imposed} flux-compatible class has neither cost: $\mathcal B\equiv0$, the class contains $s^\star$ (Proposition~\ref{prop:nomisspec}), and the excess risk reduces to the four standard terms of Proposition~\ref{prop:excess}. Pinning any trace makes ISM faithful within its class; only the no-flux trace makes the class contain the truth.

For a pinned class the boundary-bias term of Proposition~\ref{prop:excess} is vacuous --- $n^\top A_ts$ is $\theta$-independent on $\bdry$, so $\sup_{s,s'\in\HH}|\mathcal B(s)-\mathcal B(s')|=0$ --- yet Experiment~1 shows such classes carry large score error. The floor locates that cost where it actually lives: in the approximation term, bounded \emph{below} on any class that pins the wrong trace.

\begin{proposition}[Trace-misspecification floor]
\label{prop:floor}
Let $G$ be a hyperrectangle and $F$ one of its faces, with $(d-1)$-dimensional measure $|F|$ and inward slab $S_\delta=\{x\in G:\dist(x,F)\le\delta\}$ ($\mathrm{vol}(S_\delta)=|F|\,\delta$). Assume, with constants uniform in $t$:
(i) $a_0I\preceq A_t(x)\preceq a_1I$ on $\bar G$;
(ii) every $s\in\HH$ satisfies $n^\top A_ts-\beta_A\equiv\Delta_F$ on $F$, with $|\Delta_F|>0$ (for the diagonal-$A$ OU families of Appendix~\ref{app:ctr-family}, $\Delta_F^2$ is the listed per-face CTR);
(iii) $\rho_t\ge\rho_{\min}>0$ on $S_{\delta_0}$;
(iv) $e_t:=s(t,\cdot)-s^\star(t,\cdot)$ is $L$-Lipschitz on $S_{\delta_0}$ for every $s\in\HH$ (for $L=0$ read $|\Delta_F|/(2a_1L)=+\infty$).
Then, with $\delta_\star:=\min(\delta_0,|\Delta_F|/(2a_1L))$, every $s\in\HH$ obeys
\[
\D_A(s,s^\star)\;\ge\;\rho_{\min}\,a_0\,|F|\,\delta_\star\,\frac{\Delta_F^2}{4a_1^2}\;>\;0 .
\]
\end{proposition}
\noindent In words: a wrongly pinned trace forces a positive risk floor whose driver $\Delta_F^2$ is the closed-form CTR, computable before any experiment --- a wrong boundary class cannot be trained, or fed enough data, out of its bias.

\emph{Proof.} It suffices to bound $\E_{X_t\sim\rho_t}[e_t^\top A_te_t]$ for each $t$ and average over $t\sim q$. On the face, Proposition~\ref{prop:conormal} gives $n^\top A_ts^\star=\beta_A$, so $n^\top A_te_t=\Delta_F$ there. For $y\in F$,
\[
\begin{aligned}
|\Delta_F|&=|n^\top A_t(y)\,e_t(y)|\\
&\le\|A_t(y)\,n\|\,\|e_t(y)\|\le a_1\|e_t(y)\|,
\end{aligned}
\]
since $A_t(y)$ is symmetric with $\lambda_{\max}\le a_1$ and $\|n\|=1$; hence $\|e_t\|\ge|\Delta_F|/a_1$ on $F$. Any $x\in S_{\delta_\star}$ has an orthogonal foot point $y\in F$ with $\|x-y\|\le\delta_\star$ (hyperrectangle faces are flat), so the Lipschitz bound gives
\[
\|e_t(x)\|\;\ge\;\|e_t(y)\|-L\,\delta_\star\;\ge\;\frac{|\Delta_F|}{a_1}-\frac{|\Delta_F|}{2a_1}\;=\;\frac{|\Delta_F|}{2a_1}.
\]
Integrating,
\[
\begin{aligned}
\E[e_t^\top A_te_t]
&\ge a_0\!\int_G\rho_t\|e_t\|^2dx
\ge a_0\rho_{\min}\!\int_{S_{\delta_\star}}\|e_t\|^2dx\\
&\ge a_0\rho_{\min}|F|\,\delta_\star\,\frac{\Delta_F^2}{4a_1^2}.
\end{aligned}
\]
\qed

The bound is per-face: taking the maximum over the pinned faces sharpens it, and the slabs of distinct faces may be summed once $\delta_\star$ is below half their separation. The Lipschitz constant $L$ is finite for any neural class with bounded weights and smooth activations, and enters only through the shell depth: a steeper class can localize its error more sharply near $F$, but cannot remove it.

Two corollaries tie the identity to measurement. First, CTR upper-bounds the hidden boundary bias, so it is a statistical quantity, not merely a geometric diagnostic; second, on a compatible identifiable submodel the score-risk bound becomes a parameter-error bound.

\begin{corollary}[CTR controls the boundary bias]
\label{cor:ctrrho}
Let $d\mu_\partial=q(t)\,\rho_t(x)\,dS(x)\,dt$ on $[0,T]\times\bdry$, with total mass $\mu_\partial=\E_t\!\int_{\bdry}\rho_t\,dS$ (finite under Assumption~\ref{ass:regularity}), and let $\mathrm{CTR}_\rho(s)=\int|n^\top A_ts-\beta_A|^2\,d\mu_\partial$. Then
$|\mathcal B(s)|\le\sqrt{\mu_\partial\,\mathrm{CTR}_\rho(s)}$.
\end{corollary}
\emph{Proof.} $\mathcal B(s)=\int(n^\top A_ts-\beta_A)\,d\mu_\partial$; apply Cauchy--Schwarz in $L^2(\mu_\partial)$. \qed

\begin{corollary}[Parameter estimation in identifiable submodels]
\label{cor:submodel}
Let $\HH_\Theta=\{s_\vartheta:\vartheta\in\Theta\}\subset\Sbeta$ with $s^\star=s_{\vartheta_0}\in\HH_\Theta$, and suppose the local identifiability bound $\D_A(s_\vartheta,s_{\vartheta_0})\ge c\,\|\vartheta-\vartheta_0\|^2$ holds on a neighborhood of $\vartheta_0$ containing the estimator $\hat\vartheta$. Then, in the notation of Proposition~\ref{prop:excess},
\[
\begin{aligned}
\|\hat\vartheta-\vartheta_0\|^2
&\le c^{-1}\Bigl[2\sup_{s\in\HH_\Theta}
  |\Lhat_{N,\Delta}(s)-\mathcal L_\Delta(s)|\\
&\qquad+2\sup_{s\in\HH_\Theta}
  |\mathcal L_\Delta(s)-\mathcal L(s)|
  +\varepsilon_{\mathrm{opt}}\Bigr].
\end{aligned}
\]
\end{corollary}
\emph{Proof.} Since $s^\star\in\HH_\Theta$, the approximation term of Proposition~\ref{prop:excess} vanishes; chain $c\|\hat\vartheta-\vartheta_0\|^2\le\D_A(s_{\hat\vartheta},s^\star)$ with the remaining terms. \qed

We state Corollary~\ref{cor:submodel} for trace-compatible classes only. For a pinned class with the wrong trace, $s_{\vartheta_0}\notin\HH_\Theta$: the estimator concentrates near a pseudo-true parameter instead, and Proposition~\ref{prop:floor} lower-bounds the score risk at any such parameter, so no consistency statement around $\vartheta_0$ is available in that case.

\subsection{Closed-form CTR family across architectural classes}
\label{app:ctr-family}

A direct corollary of Propositions~\ref{prop:conormal} and~\ref{prop:hyperrect} is that the conormal trace each architectural class enforces on $\bdry$ is a fixed function of the class, the drift, and (where relevant) $\nabla\cdot A_t$. For the box $[-1,1]^d$ with diagonal state-dependent diffusion $A_\kappa(x)=\diag(1+\kappa x_i)$ ($|\kappa|<1$) and drift $b(x)=-\lambda x$, routine computation gives the face-average squared trace residual
{\small
\[
\begin{aligned}
&\mathrm{CTR}_{\mathrm{correct}}=0,\quad
\mathrm{CTR}_{\mathrm{omit}}=\kappa^2,\\
&\mathrm{CTR}_{\substack{\mathrm{normal}\\\mathrm{zero}}}
=\lambda^2+\kappa^2,\\
&\mathrm{CTR}_{\substack{\mathrm{wrong}\\\mathrm{sign}}}
=4(\lambda^2+\kappa^2).
\end{aligned}
\]
}
On the simplex $\Delta^{K-1}$ with $A=aP$ and OU drift toward the barycenter, the analogous computation (Appendix~\ref{app:simplex-impl}) gives
{\small
\[
\begin{aligned}
&\mathrm{CTR}_{\mathrm{flux}}=0,\\
&\mathrm{CTR}_{\mathrm{zero\text{-}flux}}
=\mathrm{CTR}_{\mathrm{full\text{-}boundary\text{-}zero}}=\lambda^2/K^2,\\
&\mathrm{CTR}_{\mathrm{wrong\text{-}sign}}=4\lambda^2/K^2,
\end{aligned}
\]
in the coordinate-CTR convention defined below. Every alternative class is therefore fixed at its own non-zero deterministic value. Because the face-vanishing factors $(1-x_i^2)$ on the box and $M_\Delta(x)$ on the simplex (Appendix~\ref{app:simplex-impl}) algebraically eliminate the trainable backbone $h_\theta$ from $v\big|_{\bdry}$, the on-boundary value of $v$ does not depend on $\theta$ at all and CTR is architectural by construction, not trainable. Section~\ref{sec:exp-trace-law} verifies this directly: across the 280 state-dependent-diffusion training runs (140 at $d=2$, 140 at $d=4$), the post-training measured CTR equals its closed-form predicted value to floating-point precision in every cell, with zero seed variance for every baseline and every $\kappa$, regardless of the magnitude the loss reaches during training (the four baselines' ISM loss trajectories span roughly $-8$ to $+17$ across the 280 cells: $-4$ to $+8$ at $d=2$ and $-8$ to $+17$ at $d=4$; the absolute scale grows with $d$ because the box ISM objective accumulates contributions from more faces).

\paragraph{Tangential boundary error and the unit-normal convention.} The tangential boundary error $\mathrm{TBE}_A(s)=\E_t\!\int_{\bdry}\|P_T^A(s-s^\star)\|_A^2\,d\nu_t$, with $A$-tangent projection $P_T^A z=z-n\,(n^\top A z)/(n^\top A n)$, measures whether a method suppresses the tangential components the no-flux condition leaves free. The CTR formulas above use the \emph{coordinate-CTR convention}: on the box (axis-aligned faces, $\|n\|=1$) the coordinate-normal trace $v_i=(As)_i$ and the unit-normal trace $n^\top As$ coincide. On the simplex $\Delta^{K-1}$ the intrinsic outward unit normal of the face $\{x_i=0\}$ inside the plane $\mathbf 1^\top x=1$ is $\nu_i=-Pe_i/\|Pe_i\|$ with $\|Pe_i\|^2=(K-1)/K$, so the unit-normal CTR is a factor $K/(K-1)$ larger than the coordinate CTR; we report coordinate CTR throughout, and the unit-normal value follows by multiplying by $K/(K-1)$ face-wise.

\section{Experimental details and additional results}

\subsection{Experiment-code index}
The main text refers to experiments by number; internally the runs carry the short codes used in the appendix headers and the supplementary code, collected here. E1 = box trace characterization, isotropic and state-dependent (Section~\ref{sec:exp-trace-law}); A1/A2 = isotropic and anisotropic conormal ablations; A3 = boundary-hugging leakage; E3 = estimator-variance comparison; E4 = five-mode GMM; E5a/E5b = simplex $K=3$ / moderate-$K$; E6 = state-dependent-box CTR; T4 = dimension stress; P1 = $(\kappa,m)$ phase diagram (Section~\ref{sec:exp-lowK}); P2 = real-design posterior (Section~\ref{sec:exp-realdesign}); P3 = SDF polygon (Section~\ref{sec:exp-sdf}); P5 = single-active-face negative controls. There is no P4.

\paragraph{Computing environment.} The reported GPU experiments were run on a local Windows workstation with an NVIDIA GeForce RTX 5060 Ti (16 GB). Per-run environment snapshots in the code-and-data package record Python 3.12.13 and PyTorch 2.11.0+cu128; commands, configurations, and seeds are stored with the corresponding run artifacts.

\subsection{E1 full main table}
\label{app:e1-full}
The primary cell $(\lambda,\epsilon)=(3,0.5)$, with all 7 method classes including the exact-score oracle and unconstrained-ISM strawman, is reported below. Means $\pm$ standard deviations are over 5 seeds.

\begin{table*}[t]
\centering
\small
\begin{tabular}{lrrrr}
\toprule
method & bulk & shell$_{0.02}$ & normal trace & tangential \\
\midrule
oracle (true score) & 0 & 0 & 1.89e--14 $\pm$ 0 & 0 \\
unconstrained ISM & 54.2 $\pm$ 2 & 4650 $\pm$ 100 & 37\,900 $\pm$ 1000 & 278 $\pm$ 20 \\
unconstr. + bndry quad. & 2.49e--03 $\pm$ 4e--04 & 5.29e--02 $\pm$ 1e--02 & 5.19e--02 $\pm$ 8e--03 & 3.38e--02 $\pm$ 7e--03 \\
zero-boundary (Fishman) & 0.163 $\pm$ 0.001 & 7.17 $\pm$ 0.02 & 9.00 $\pm$ 0 & 1.23 $\pm$ 0 \\
normal-zero & 0.107 $\pm$ 0.001 & 6.31 $\pm$ 0.02 & 9.00 $\pm$ 0 & 0.516 $\pm$ 0.01 \\
wrong-sign trace & 0.376 $\pm$ 0.003 & 24.1 $\pm$ 0.1 & 36.00 $\pm$ 0 & 1.67 $\pm$ 0.07 \\
\textbf{flux-compatible (ours)} & \textbf{5.42e--04} $\pm$ 5e--05 & \textbf{2.06e--03} $\pm$ 7e--04 & \textbf{0} & \textbf{2.39e--03} $\pm$ 8e--04 \\
\bottomrule
\end{tabular}
\caption{E1 full table at $(\lambda,\epsilon)=(3,0.5)$, 5 seeds.}
\label{tab:e1-full}
\end{table*}

\begin{figure}[t]
\centering
\includegraphics[width=0.9\linewidth]{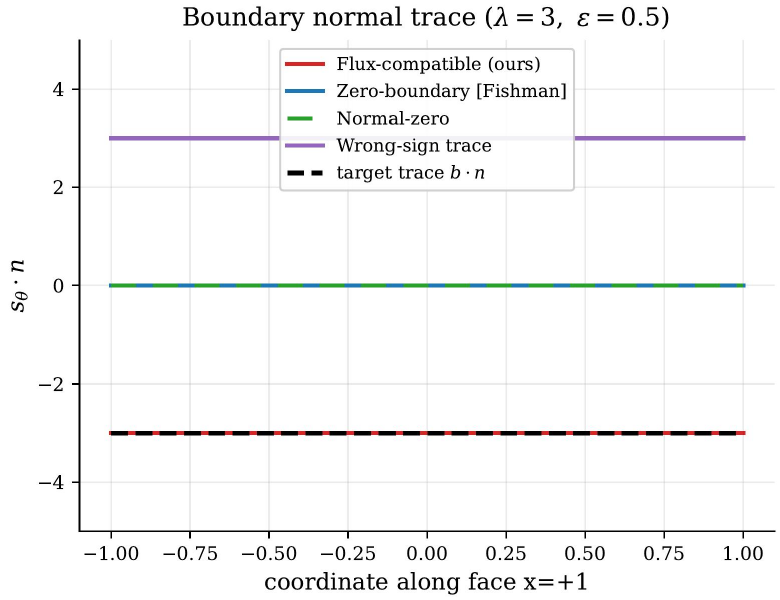}
\caption{Only the flux-compatible construction matches the target normal trace $b\cdot n=-3$. E1 normal trace $s_\theta\cdot n$ along the face $x=+1$ at $(\lambda,\epsilon)=(3,0.5)$ (target dashed). Boundary-zero and normal-zero pin to $0$; wrong-sign pins to $+3$.}
\label{fig:normal-trace-app}
\end{figure}

\begin{figure*}[t]
\centering
\includegraphics[width=0.92\textwidth]{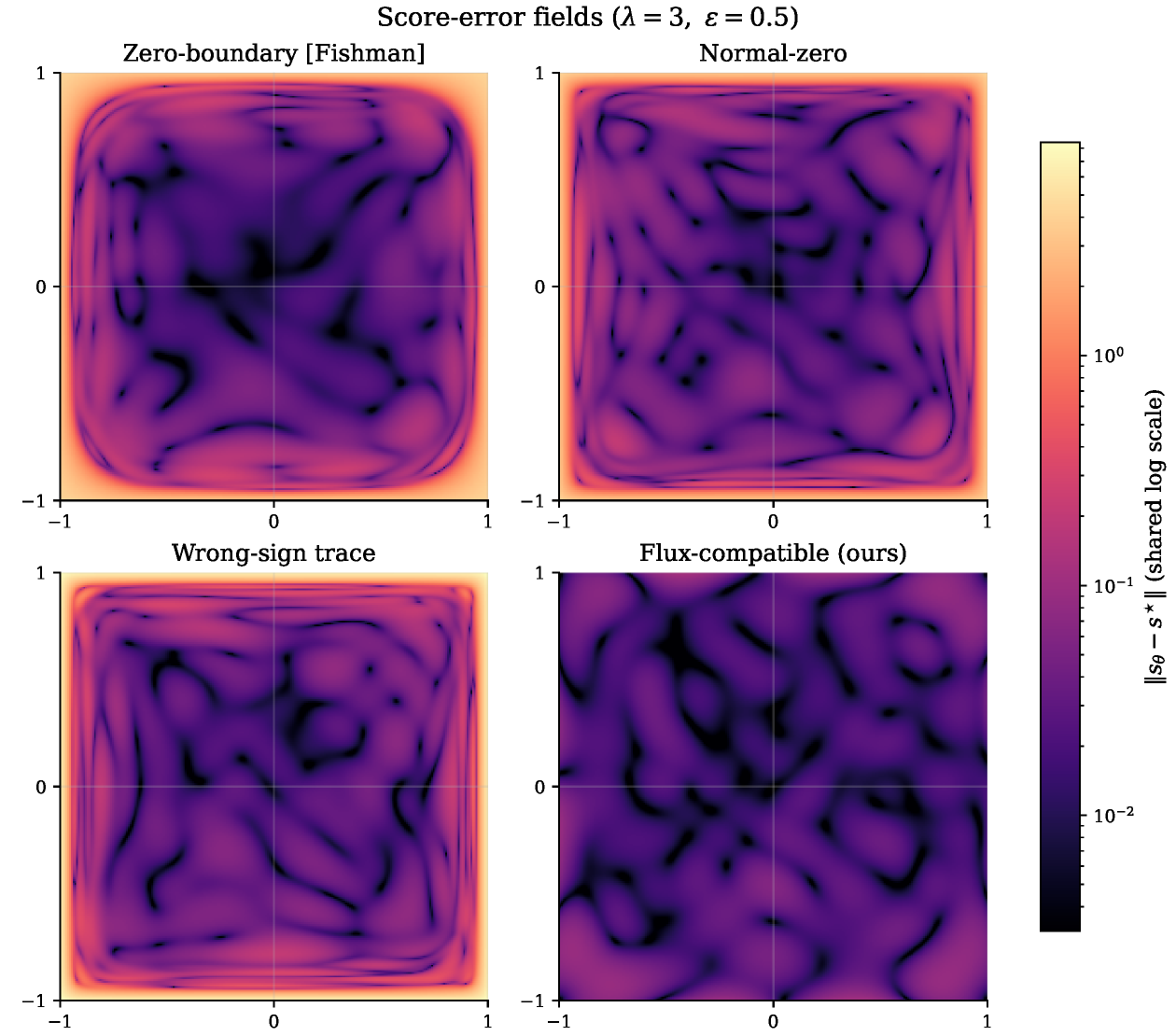}
\caption{Trace-violating classes show bright boundary error halos; the flux-compatible field is uniformly dark up to the boundary. Score-error magnitude fields $\|s_\theta-s^\star\|$ on $[-1,1]^2$ at $(\lambda,\epsilon)=(3,0.5)$, shared log scale (referenced from Section~\ref{sec:exp-trace-law}).}
\label{fig:error-fields}
\end{figure*}

\begin{figure}[t]
\centering
\includegraphics[width=0.9\linewidth]{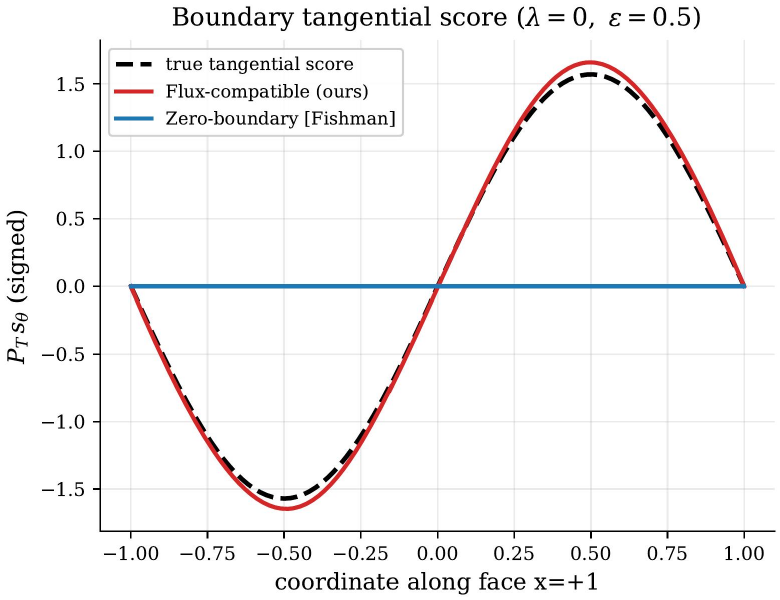}
\caption{Zero-boundary parametrizations erase the tangential component the no-flux condition leaves free; the flux-compatible construction retains it. Tangential boundary trace at $(\lambda,\epsilon)=(0,0.5)$, where the normal trace is zero but the true tangential boundary score is nonzero (referenced from Section~\ref{sec:exp-trace-law}).}
\label{fig:tangential}
\end{figure}

\subsection{Full $(\lambda,\epsilon)$ phase grid}
\label{app:phase-grid}
The full $(\lambda,\epsilon)$ analytic grid (4 reflection strengths $\lambda\in\{0,1,3,10\}$ $\times$ 3 anisotropies $\epsilon\in\{0,0.2,0.5\}$ $\times$ 7 model classes $\times$ 5 seeds = 420 runs) underlies Figures~\ref{fig:phase-grid}--\ref{fig:phase-advantage}. In drift-dominated cells $(\lambda\ge 1,\epsilon=0)$, the flux-compatible class has the largest advantage (2.0--2.8 orders of magnitude); when $\epsilon>0$, tangential and oscillatory approximation errors reduce the relative gap to 0.3--0.7 orders.

\begin{figure*}[t]
\centering
\includegraphics[width=\textwidth]{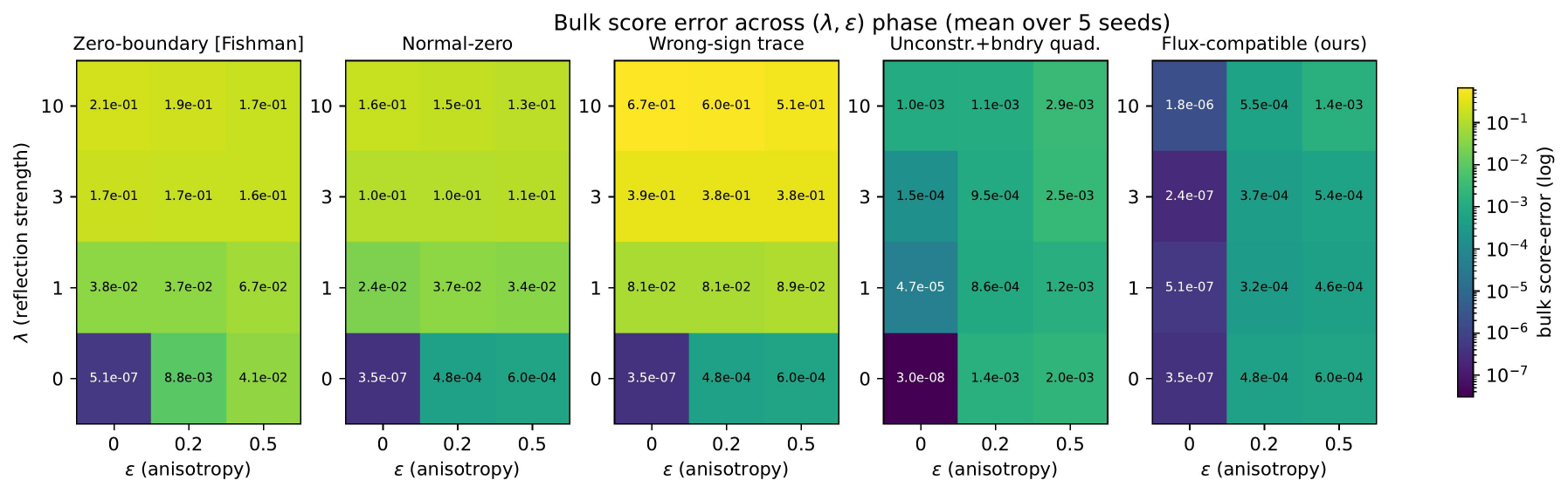}
\caption{Trace-violating classes saturate at $10^{-1}$--$10^{0}$ for $\lambda\ge 1$, while flux-compatible and the quadrature oracle stay at or below $10^{-3}$. Bulk score-matching error $\E\|s_\theta-s^\star\|^2$ on the full $(\lambda,\epsilon)$ grid, mean over 5 seeds, shared log color scale.}
\label{fig:phase-grid}
\end{figure*}

\begin{figure}[t]
\centering
\includegraphics[width=0.7\linewidth]{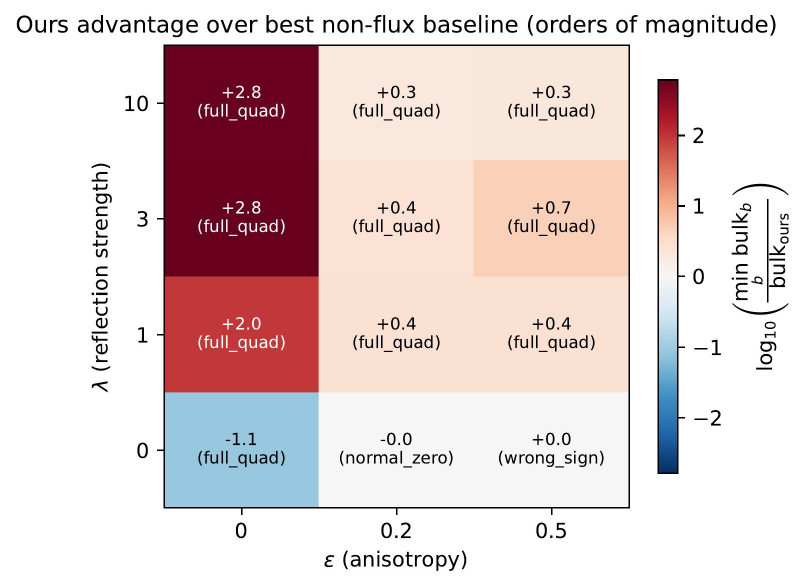}
\caption{Positive (red) cells mark where flux-compatible beats the lowest-error non-flux baseline in that cell. Per-cell ratio $\log_{10}\!\left(\min_{\text{baseline}}\mathrm{bulk}_b/\mathrm{bulk}_{\mathrm{ours}}\right)$ over the same $(\lambda,\epsilon)$ grid.}
\label{fig:phase-advantage}
\end{figure}

\subsection{A2 anisotropic conormal full table}
\label{app:a2-full}

On $[-1,1]^2$ with drift $b=-\lambda x$, the $A$-compatible oracle is
$\log\rho_{\lambda,\epsilon,A}(x,y)=-\tfrac\lambda2 x^\top A^{-1}x+
\epsilon\cos(\pi x)\cos(\pi y)-\log Z_A$. The oscillatory normal derivative
vanishes on every face, so $n^\top A\nabla\log\rho=n^\top b$ and the conormal
target is $\beta_A=-\lambda$; $A=I$ recovers the isotropic companion.

\begin{table*}[t]
\centering
\small
\begin{tabular}{lrrrr}
\toprule
method & predicted CTR & measured CTR & score err (A) & shell$_{0.05}$ (A) \\
\midrule
\textbf{correct conormal} & 0 & \textbf{0 $\pm$ 0} & \textbf{2.79e--3 $\pm$ 3.9e--4} & \textbf{3.09e--3 $\pm$ 6.1e--4} \\
wrong normal-only & 40.5 & 40.50 $\pm$ 0 & 2.13e--1 $\pm$ 9.2e--3 & 4.195 $\pm$ 2.0e--1 \\
zero-boundary & 9 & 9.00 $\pm$ 0 & 2.90e--1 $\pm$ 6.5e--3 & 3.959 $\pm$ 3.3e--2 \\
wrong-sign conormal & 36 & 36.00 $\pm$ 0 & 3.92e--1 $\pm$ 4.3e--3 & 7.667 $\pm$ 8.2e--2 \\
\bottomrule
\end{tabular}
\caption{A2 anisotropic conormal test, $A=\diag(1,4)$, $\lambda=3$, $\epsilon=0.5$, mean $\pm$ std over 5 seeds.}
\label{tab:a2-full}
\end{table*}

\subsection{E3 estimator-variance details}
\label{app:e3-full}
The reported gradient covariance trace is computed with a mini-batch of $J=256$ samples and reported as the mean over 5 seeds at $\lambda=3,\epsilon=0.5$. The local-time estimator uses one-step reflected-Euler overshoots calibrated against the analytic boundary integral; the exact-quadrature oracle uses an analytic surface integral on $[-1,1]^2$. For the unconstrained-architecture groups, the network is initialized from a converged \texttt{full\_quad} checkpoint to suppress interior variance and isolate the estimator contribution. A complementary $\lambda$-sweep shows variance ratios between the local-time estimator and exact quadrature of $\approx 900\times$ at $\lambda=1$, falling to $43\times$ at $\lambda=10$ ($915\times$, $537\times$, $42.7\times$ at $\lambda=1,3,10$). The boundary-gradient ratio $R_\partial$ for ours plus local time is $0$ for every $h$, confirming that a trace-compatible score nullifies the trainable boundary term as predicted by Proposition~\ref{prop:ism}.

\subsection{Parameter recovery in the identifiable submodel (Figure~\ref{fig:stat-triptych}a)}
\label{app:param-recovery}
We use the $A{=}I$ analytic density $\log\rho_\varphi(x,y)=-\tfrac\lambda2\|x\|^2+\varphi\,c(x,y)-\log Z_\varphi$ with $c=\cos(\pi x)\cos(\pi y)$, drift $b=-\lambda x$ held \emph{fixed and known} (making $\varphi$ the only unknown; letting $\lambda$ vary would move the trace target and reintroduce a non-constant boundary term), and truth $\varphi_0=0.5$, $\lambda=3$. The true score is $s^\star=b+\varphi_0 k$ with $k=\nabla c$; because $n^\top k=0$ on every face (as $\sin(\pm\pi)=0$), the compatible family $s_\varphi=b+\varphi k$ lies in $\Sbeta$ for all $\varphi$ and contains the truth, while the Fishman family $s^{\mathrm F}_\varphi=(1-x^2)(1-y^2)(b+\varphi k)$ cannot. Both are affine in $\varphi$, so the empirical interior-ISM objective is an exact quadratic and the estimator is closed-form, $\hat\varphi=-\widehat{\E}[w\cdot u+\nabla\!\cdot\!u]/\widehat{\E}[\|u\|^2]$ with $s=w+\varphi u$. By Gauss--Legendre quadrature the compatible population minimizer equals $\varphi_0$ exactly, whereas the Fishman pseudo-true amplitude is $\varphi^\star_{\mathrm F}=0.760$ (bias $+0.260$), with $A$-weighted risk floor $\min_\varphi\D_A(s^{\mathrm F}_\varphi,s^\star)=1.38$. Over $n\in\{5\!\times\!10^2,\ldots,10^5\}$ and 20 seeds (paired data), the compatible MSE tracks $\mathrm{avar}_c/n$ ($\mathrm{avar}_c=4.92$) and the Fishman MSE plateaus at $\mathrm{bias}^2$; the run is pure-numpy and deterministic.

\subsection{Sample-size sweep and the finite-$N$ stabilizer (Figure~\ref{fig:stat-triptych}b,c)}
\label{app:n-sweep}
At the primary cell $(\lambda,\epsilon)=(3,0.5)$ on $[-1,1]^2$ we sweep the training-set size $N\in\{5\!\times\!10^2,10^3,5\!\times\!10^3,2\!\times\!10^4,10^5\}$ with a fixed i.i.d.\ data bank drawn once per $(N,\text{seed})$; each minibatch resamples from the bank, so $N$ is the statistical sample size (the streaming experiments elsewhere correspond to $N=\infty$). All four boundary classes share the bank, initialization, and batch sequence per seed (paired), 5 seeds, $2\!\times\!10^4$ steps. On a fixed bank the empirical interior-ISM objective is unbounded below --- the optimizer digs divergence wells at the bank points, the finite-data form of the moderate-$K$ pathology of Appendix~\ref{app:simplex-pathology} --- so we apply the same \emph{non-capping} Jacobian stabilizer $\lambda_{\mathrm{Jac}}\E\|\nabla s\|_F^2$ as Experiment~4 ($\lambda_{\mathrm{Jac}}=10^{-2}$, one Hutchinson probe, identical for every baseline; a dose sweep $\{0,10^{-3},10^{-2}\}$ and the un-stabilized pathology run are in the supplementary material). The boundary-shell risk (Figure~\ref{fig:stat-triptych}b) descends from $32.0$ to $0.12$ for flux-compatible and plateaus at $6.4/6.0/22.8$ for zero-boundary/normal-zero/wrong-sign (CTR $9/9/36$), a $50$--$190\times$ separation at $N=10^5$; leakNone and ProjDisp separate identically ($0.16/0.04$ vs.\ $0.83$--$1.00/0.61$--$1.40$). Throughout, and even in the un-stabilized blow-up, the measured CTR stays exactly $0/9/9/36$ with zero seed variance (Figure~\ref{fig:stat-triptych}c), isolating the structural boundary guarantee from the interior objective.

\subsection{E4 GMM independent-seed table}
\label{app:e4-full}

\begin{table*}[t]
\centering
\small
\begin{tabular}{lrrrr}
\toprule
baseline & $\lambda=0$ & $\lambda=1$ & $\lambda=3$ & $\lambda=10$ \\
\midrule
ours & 1.27e--2 $\pm$ 2.5e--3 & \textbf{8.64e--3 $\pm$ 1.7e--3} & \textbf{1.15e--2 $\pm$ 4.0e--3} & 2.92e--2 $\pm$ 8.3e--3 \\
normal-zero & 9.98e--3 $\pm$ 1.4e--3 & 9.92e--3 $\pm$ 2.6e--3 & 1.18e--2 $\pm$ 3.6e--3 & 2.55e--2 $\pm$ 5.4e--3 \\
Fishman & 1.67e--2 $\pm$ 6.8e--3 & 1.70e--2 $\pm$ 3.8e--3 & 1.77e--2 $\pm$ 2.8e--3 & 2.55e--2 $\pm$ 5.2e--3 \\
wrong-sign & 1.23e--2 $\pm$ 1.8e--3 & 1.55e--2 $\pm$ 6.2e--3 & 3.23e--2 $\pm$ 1.3e--2 & 4.16e--2 $\pm$ 2.2e--2 \\
unconstrained & 0.156 $\pm$ 9.0e--2 & 0.115 $\pm$ 2.3e--2 & 0.170 $\pm$ 3.7e--2 & 3.55e--2 $\pm$ 1.3e--2 \\
unconstrained + lt & 0.202 $\pm$ 4.2e--2 & 1.88e--2 $\pm$ 3.9e--3 & 2.54e--2 $\pm$ 5.8e--3 & 2.78e--2 $\pm$ 8.6e--3 \\
\bottomrule
\end{tabular}
\caption{E4 2D 5-mode GMM independent-seed results: MMD$^2$ to data, mean $\pm$ std over 5 independent per-baseline seeds. Trace error is sampling-distribution-free. At $\lambda=0$ the prescribed conormal trace vanishes, so ours, normal-zero, and wrong-sign reduce to the same trace-compatible architecture (the baseline formulas of Appendix~\ref{app:baseline-formulas}).}
\label{tab:e4-full}
\end{table*}

Trace error for the GMM experiment: ours has $0$ for every $\lambda$; normal-zero and Fishman scale as $\lambda^2$; wrong-sign as $4\lambda^2$; unconstrained diverges ($\sim 10^3$); unconstrained+lt is $0.33$--$1.63$ for $\lambda\ge 1$ (and degenerates at $\lambda=0$ because the local-time calibration vanishes).

\subsection{A3 boundary-hugging full leakage tables}
\label{app:a3-full}

This appendix collects the full boundary-hugging leakage results behind the fivefold separation quoted in Section~\ref{sec:exp-lowK}: the fold-versus-no-fold dissociation (Figure~\ref{fig:fold-vs-none}), the five-seed no-reflection leakage table at two drift strengths (Table~\ref{tab:a3-full}), and the per-baseline scatter grids that show why hard-reflected metrics look alike while leakNone separates the classes.

\begin{figure*}[t]
\centering
\includegraphics[width=0.92\textwidth]{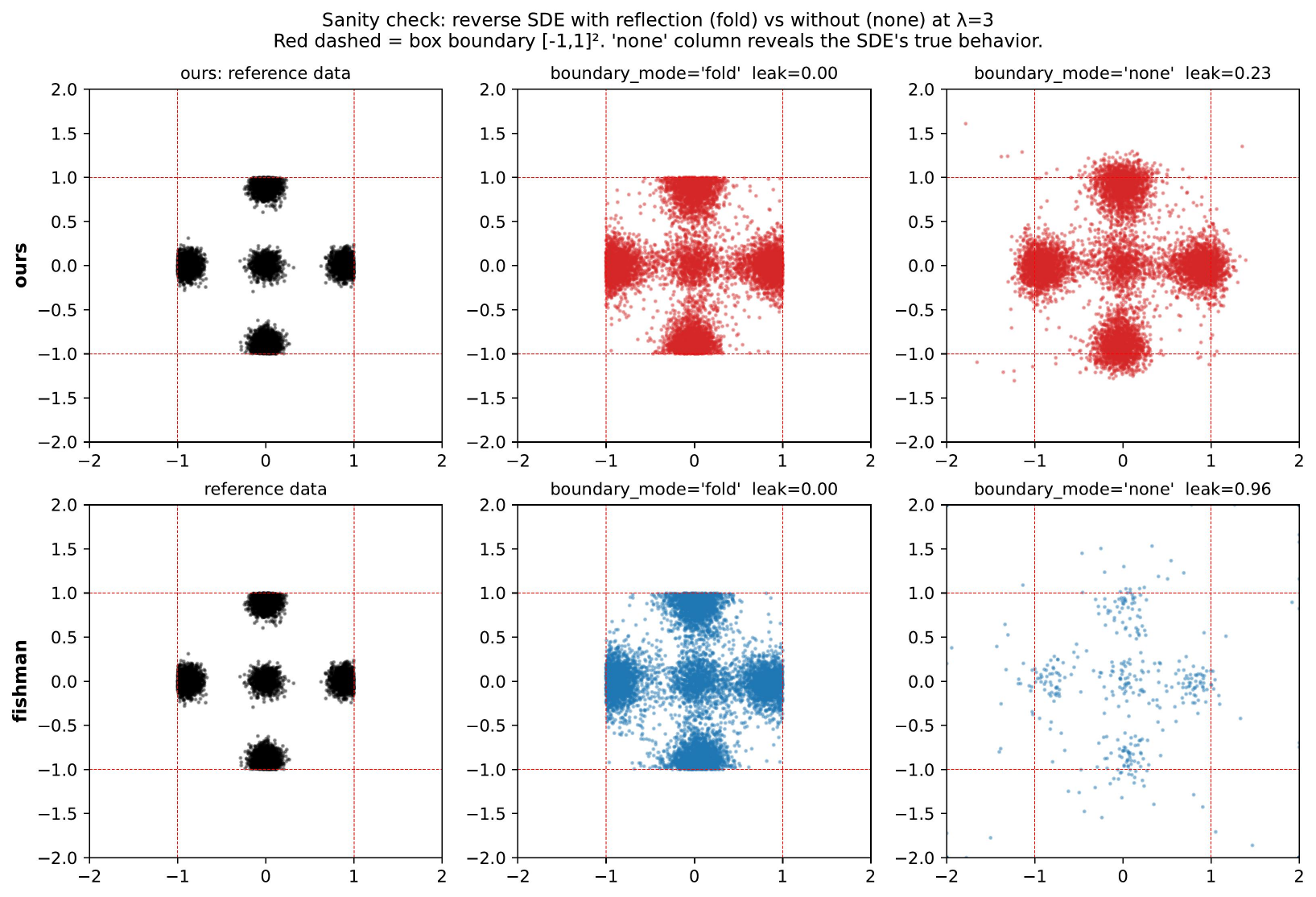}
\caption{Hard reflection masks a wrong boundary trace: with the fold, ours and Fishman look alike; without it, Fishman leaks while ours does not. Boundary-hugging GMM at $(d=2,\,\lambda=3,\,d_{\mathrm{bdry}}=0.05)$. Top row (ours): data; reverse SDE with fold (leak fraction $0.000$); without fold ($0.223$, matching Table~\ref{tab:a3-full}). Bottom row (Fishman): same panels, no-fold leak $0.964$ (seed 0; the Table~\ref{tab:a3-full} five-seed mean is $0.962$). Domain boundary $\pm 1$ dashed; everything outside is a leaked particle.}
\label{fig:fold-vs-none}
\end{figure*}

\begin{table*}[t]
\centering
\small
\begin{tabular}{llrrrr}
\toprule
$\lambda$ & baseline & $d_{\mathrm{bdry}}=0.30$ & $0.20$ & $0.10$ & $0.05$ \\
\midrule
3 & ours & \textbf{0.185 $\pm$ 0.003} & \textbf{0.191 $\pm$ 0.006} & \textbf{0.215 $\pm$ 0.004} & \textbf{0.223 $\pm$ 0.014} \\
3 & normal-zero & 0.986 $\pm$ 0.002 & 0.987 $\pm$ 0.003 & 0.989 $\pm$ 0.002 & 0.989 $\pm$ 0.002 \\
3 & Fishman & 0.927 $\pm$ 0.009 & 0.948 $\pm$ 0.007 & 0.956 $\pm$ 0.009 & 0.962 $\pm$ 0.003 \\
3 & wrong-sign & 0.991 $\pm$ 0.002 & 0.993 $\pm$ 0.003 & 0.993 $\pm$ 0.001 & 0.994 $\pm$ 0.001 \\
\midrule
10 & ours & \textbf{0.019 $\pm$ 0.004} & \textbf{0.047 $\pm$ 0.003} & \textbf{0.088 $\pm$ 0.007} & \textbf{0.104 $\pm$ 0.007} \\
10 & normal-zero & 0.632 $\pm$ 0.023 & 0.570 $\pm$ 0.014 & 0.630 $\pm$ 0.069 & 0.675 $\pm$ 0.031 \\
10 & Fishman & 0.517 $\pm$ 0.047 & 0.585 $\pm$ 0.028 & 0.572 $\pm$ 0.078 & 0.611 $\pm$ 0.048 \\
10 & wrong-sign & 0.630 $\pm$ 0.059 & 0.696 $\pm$ 0.088 & 0.751 $\pm$ 0.069 & 0.709 $\pm$ 0.057 \\
\bottomrule
\end{tabular}
\caption{No-reflection leakage on boundary-hugging GMM, mean $\pm$ std over 5 seeds (lower is better).}
\label{tab:a3-full}
\end{table*}

\begin{figure*}[t]
\centering
\includegraphics[width=\textwidth]{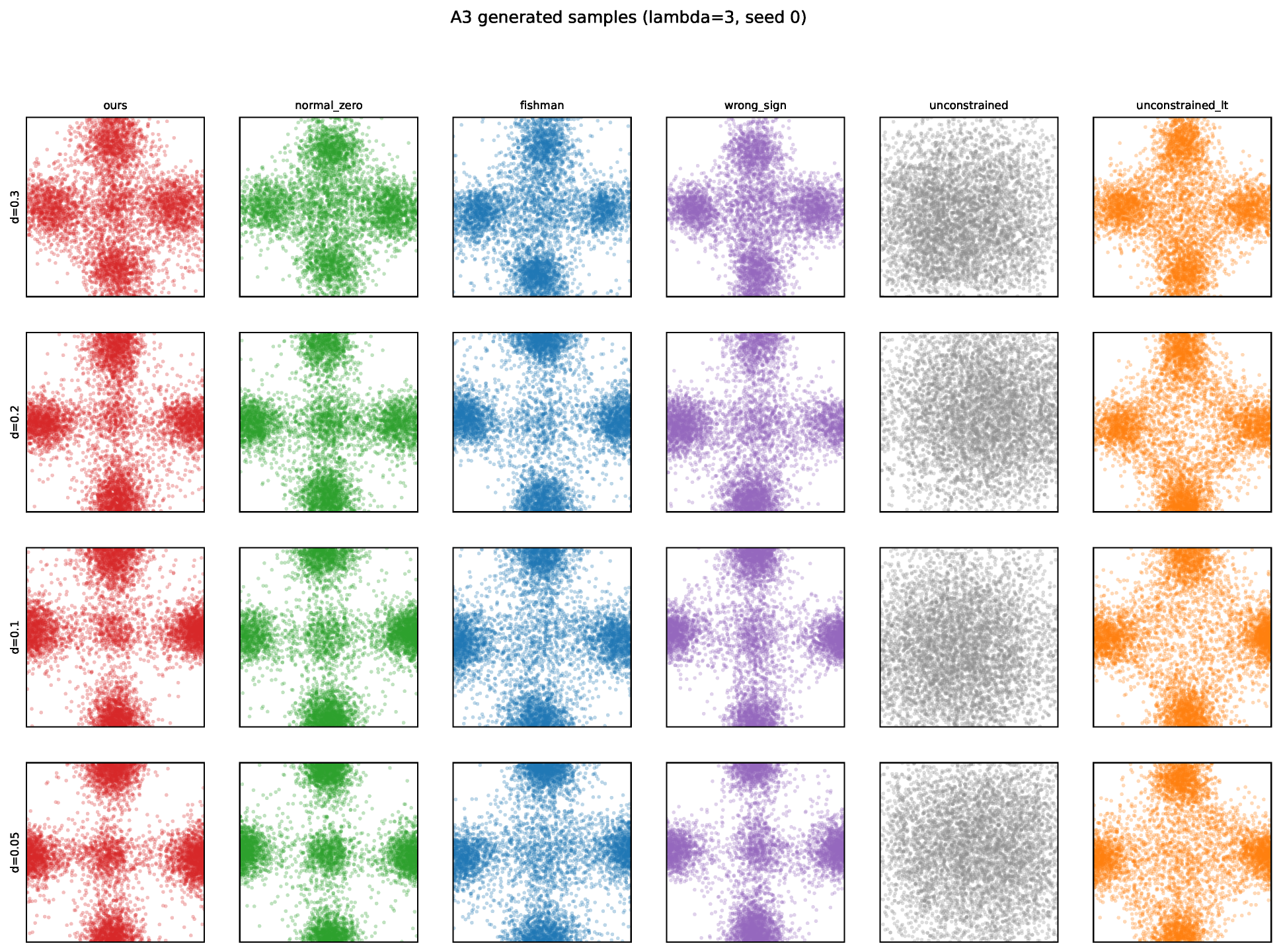}
\caption{Under hard reflection all flux-respecting methods look visually similar; no-reflection leakage exposes the difference (Table~\ref{tab:a3-full}). A3 scatter grid at $\lambda=3$, 4 values of $d_{\mathrm{bdry}}$ $\times$ 6 baselines.}
\label{fig:a3-scatter-l3}
\end{figure*}

\begin{figure*}[t]
\centering
\includegraphics[width=\textwidth]{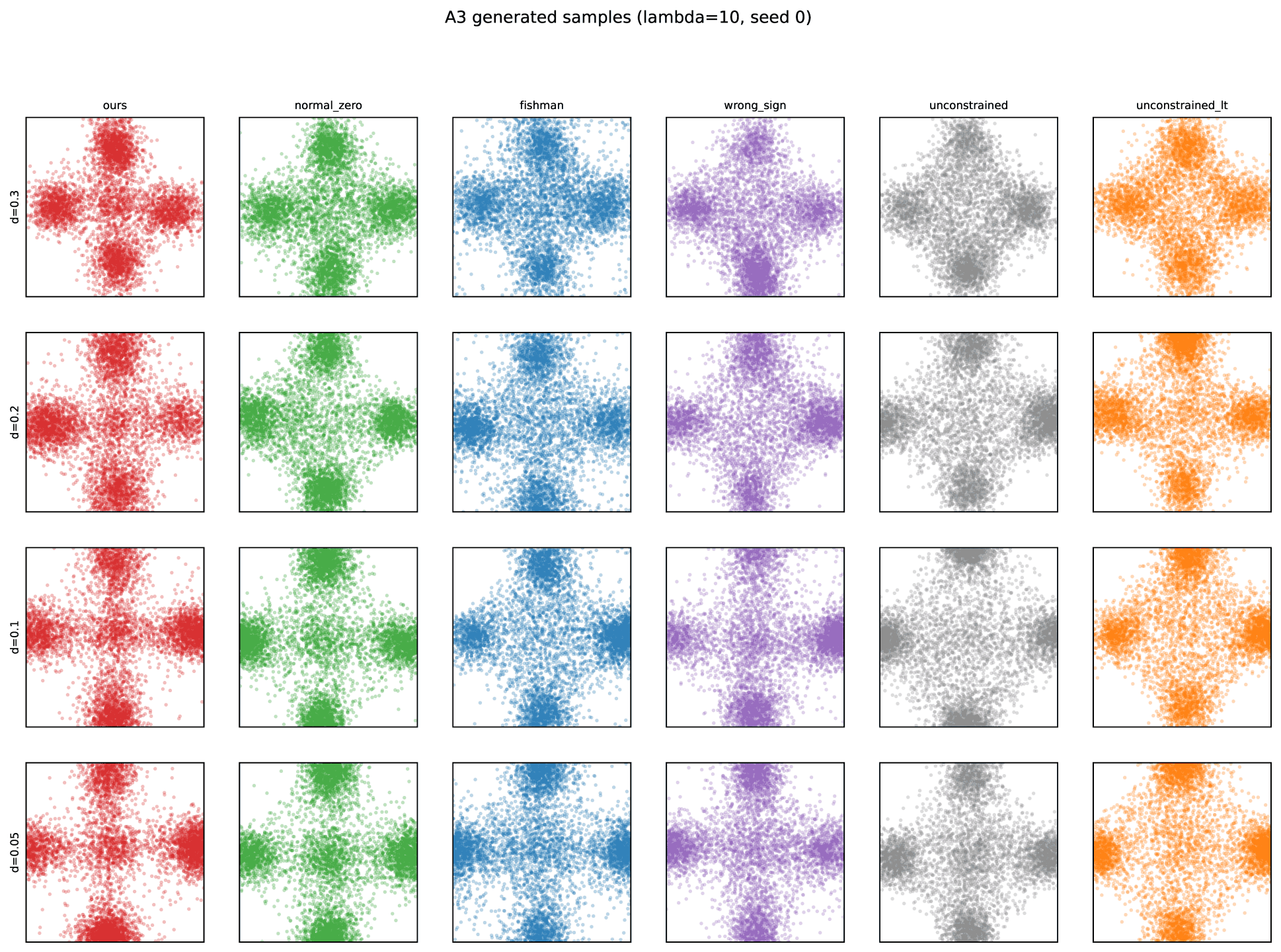}
\caption{The boundary leakage gap widens further at $\lambda=10$ (Table~\ref{tab:a3-full}). A3 scatter grid at $\lambda=10$.}
\label{fig:a3-scatter-l10}
\end{figure*}

\begin{figure*}[t]
\centering
\includegraphics[width=\textwidth]{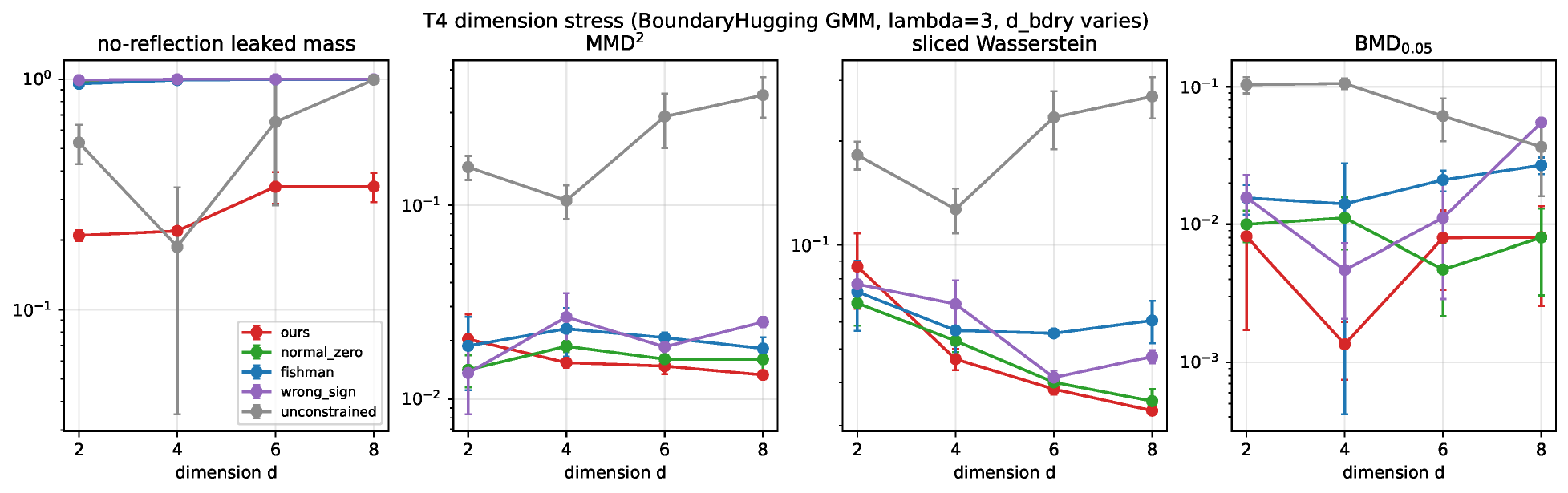}
\caption{The leakage separation persists across $d\in\{2,4,6,8\}$ while distributional metrics remain tangled. Dimension stress: boundary-hugging GMM at $\lambda=3$, $d_{\mathrm{bdry}}=0.10$. leakNone for ours remains in $[0.21,0.34]$; trace-violating baselines saturate at $\approx 0.99$--$1.00$; distributional metrics (MMD$^2$, SWD, BMD$_{0.05}$) remain tangled across methods.}
\label{fig:dim-stress}
\end{figure*}

\subsection{Flower boundary diagnostic}
\label{app:flower}
The flower dataset is generated from the polar curve $r(\theta)=0.78+0.17\sin(5\theta)$ with small radial/tangential noise ($\sigma=0.06$), a near-1D manifold embedded in 2D. At 20k training steps, visual samples from all methods remain limited and often resemble a mean-radius ring rather than the five-petal target (Figure~\ref{fig:flower-6k-20k}). We therefore treat this dataset as a diagnostic limitation, not as evidence of strong visual generation.

\begin{table}[t]
\centering
\small
\begin{tabular}{lrr}
\toprule
baseline & MMD$^2$ & leakNone \\
\midrule
ours & 4.96e--3 $\pm$ 1.1e--3 & \textbf{0.178 $\pm$ 0.009} \\
Fishman & 5.03e--3 $\pm$ 1.8e--3 & 0.961 $\pm$ 0.011 \\
normal-zero & 8.56e--3 $\pm$ 3.5e--3 & 0.988 $\pm$ 0.002 \\
wrong-sign & 1.22e--2 $\pm$ 6.2e--3 & 0.992 $\pm$ 0.003 \\
\bottomrule
\end{tabular}
\caption{Flower diagnostic at 20k steps and $\lambda=3$, mean $\pm$ std over 5 seeds.}
\label{tab:flower}
\end{table}

Ours and Fishman have statistically tied MMD$^2$ while leakNone separates them by a factor of $\approx 5$ --- an instance of the masking effect described in Section~\ref{sec:disc-principle}. Ceiling diagnostics (Figure~\ref{fig:ceiling-compare}) indicate that the reverse stopping time $t_{\min}$ is not the bottleneck on this dataset; the bottleneck is score-matching saturation on the near-1D manifold.

\begin{figure*}[t]
\centering
\includegraphics[width=0.70\textwidth]{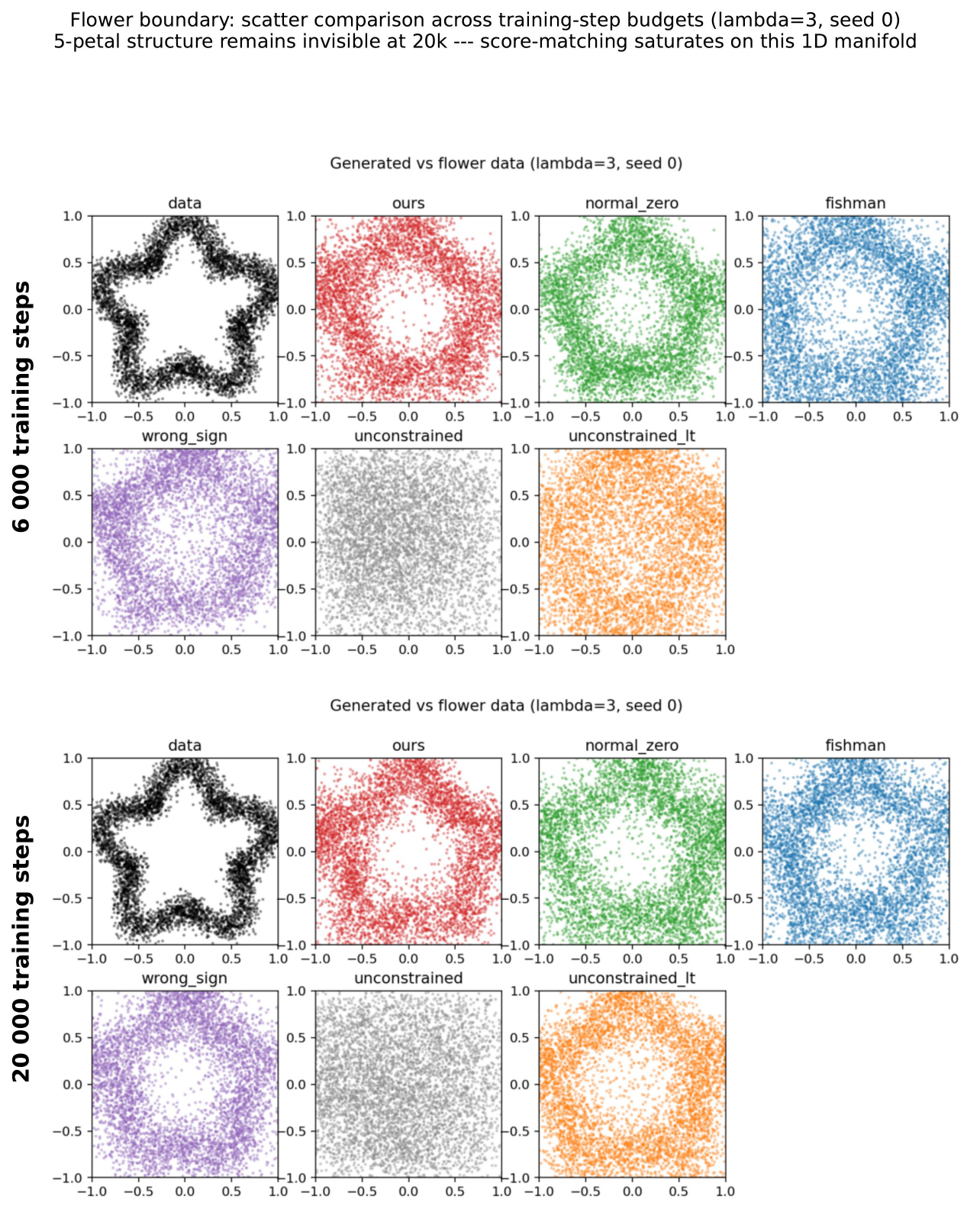}
\caption{The five-petal structure is not recovered at either training budget; all flux-respecting baselines collapse to a mean-radius ring. Flower dataset, scatter comparison at 6k and 20k training steps, $\lambda=3$, seed 0.}
\label{fig:flower-6k-20k}
\end{figure*}

\begin{figure*}[t]
\centering
\includegraphics[width=0.92\textwidth]{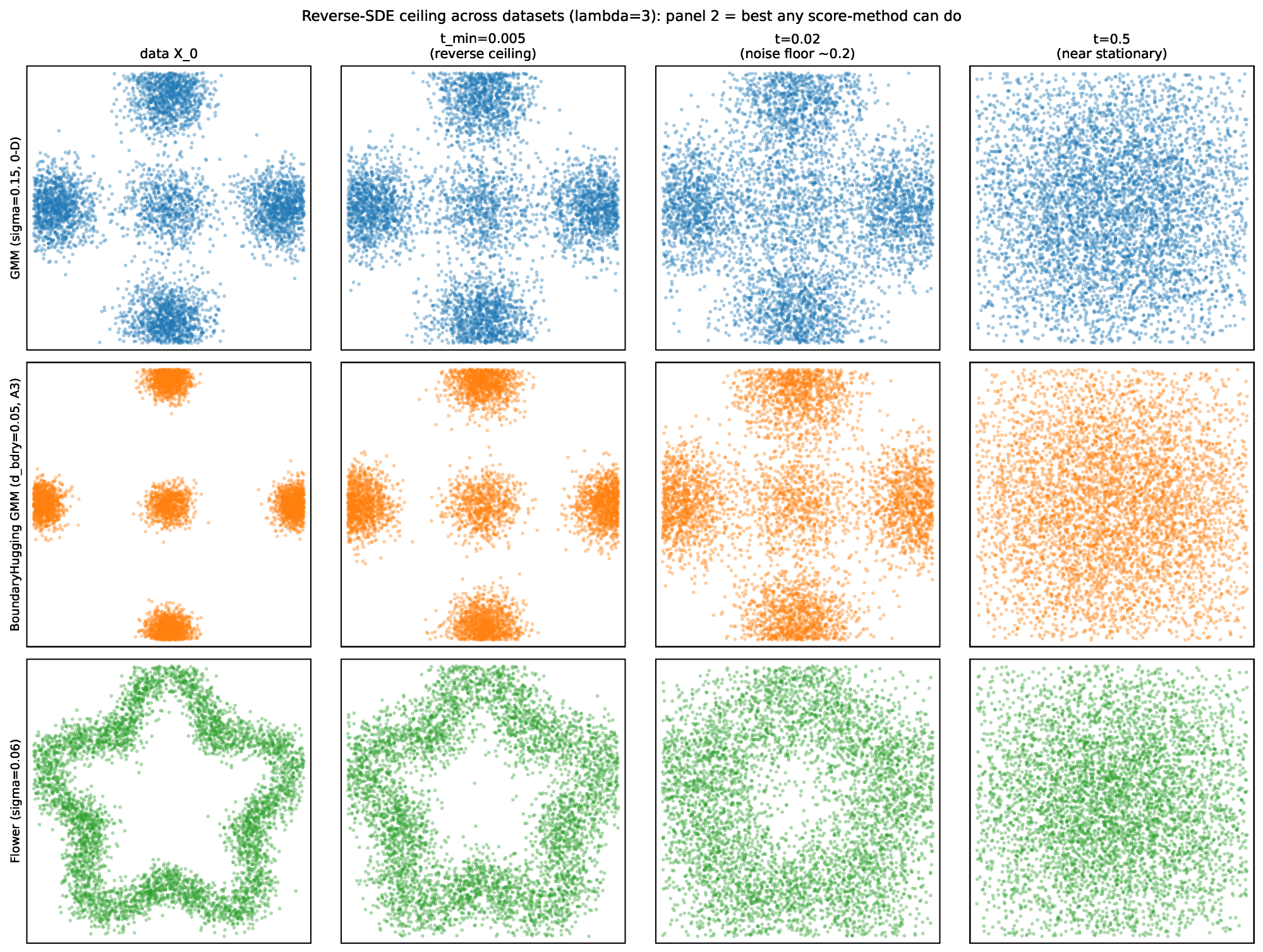}
\caption{The reverse stopping time is not the bottleneck on the flower dataset. Forward-OU snapshots at $t\in\{0,\,t_{\min}=0.005,\,0.02,\,0.5\}$ across the GMM, boundary-hugging GMM, and flower datasets; on the flower, the $t_{\min}$ snapshot still preserves the petal structure.}
\label{fig:ceiling-compare}
\end{figure*}

\begin{figure*}[t]
\centering
\includegraphics[width=\textwidth]{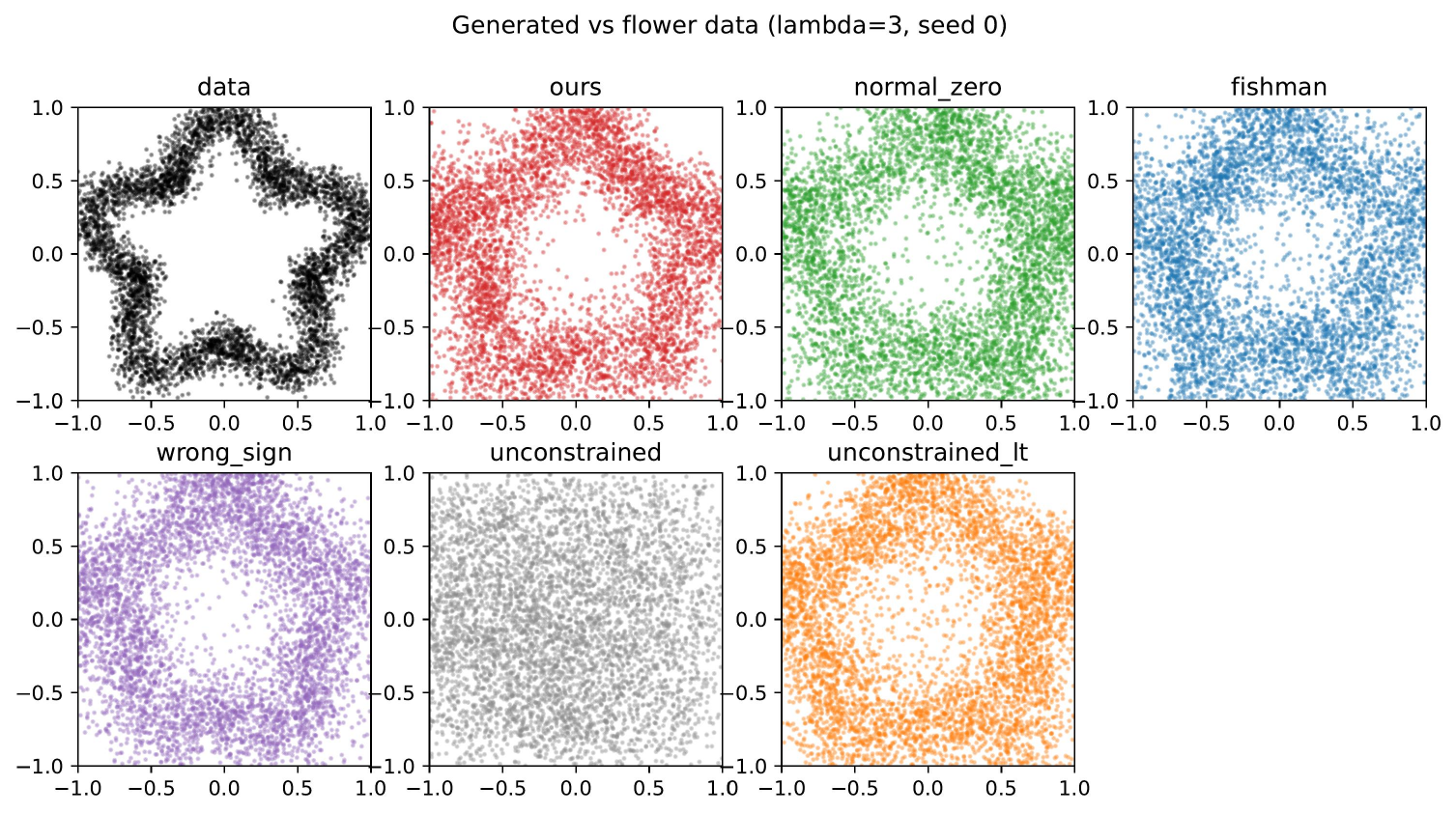}
\caption{All flux-respecting methods are visually similar; the distinguishing signal is leakNone (Table~\ref{tab:flower}). Flower 20k step, six baselines side-by-side at $\lambda=3$, seed 0.}
\label{fig:flower-baselines}
\end{figure*}

\subsection{Extreme drift ($\lambda\in\{30,100\}$)}
\label{app:extreme}

\begin{table*}[t]
\centering
\small
\begin{tabular}{llrrr}
\toprule
$\lambda$ & baseline & MMD$^2$ & trace error & leakNone \\
\midrule
30 & ours & 0.119 $\pm$ 0.027 & \textbf{0} & \textbf{0.012 $\pm$ 0.003} \\
30 & normal-zero & 0.130 $\pm$ 0.023 & 900 & 0.081 $\pm$ 0.017 \\
30 & Fishman & 0.119 $\pm$ 0.016 & 900 & 0.058 $\pm$ 0.014 \\
30 & wrong-sign & 0.106 $\pm$ 0.017 & 3600 & 0.130 $\pm$ 0.027 \\
30 & unconstrained & 0.130 $\pm$ 0.018 & 11.8 $\pm$ 1.9 & 0.013 $\pm$ 0.002 \\
\midrule
100 & ours & 0.733 $\pm$ 0.021 & \textbf{0} & \textbf{0} \\
100 & normal-zero & 0.742 $\pm$ 0.024 & 10\,000 & 0.108 $\pm$ 0.064 \\
100 & Fishman & 0.642 $\pm$ 0.022 & 10\,000 & 0.161 $\pm$ 0.025 \\
100 & wrong-sign & 0.725 $\pm$ 0.036 & 40\,000 & 0.263 $\pm$ 0.150 \\
100 & unconstrained & 0.729 $\pm$ 0.052 & 976 $\pm$ 110 & 0 \\
\bottomrule
\end{tabular}
\caption{Extreme drift, $d_{\mathrm{bdry}}=0.10$, mean $\pm$ std over 3 seeds. At $\lambda=100$ the stationary standard deviation $1/\sqrt{2\lambda}\approx 0.07$ is much smaller than the mode radius $0.9$, so MMD$^2$ is bad for every method.}
\label{tab:extreme}
\end{table*}

Two observations. First, at $\lambda=100$ the stationary distribution is too tight for the reverse SDE to reach data modes at $r=0.9$, so MMD$^2$ is uniformly poor; this is a property of the toy setup, not a method failure. Second, leakage loses discriminative power because particles rarely approach the boundary; in this regime trace error remains the right diagnostic. Note that the trace error for the \emph{unconstrained} baseline reported here ($11.8$ at $\lambda=30$, $976$ at $\lambda=100$) is much smaller than $\lambda^2$: with no architectural boundary handling, the MLP picks up a residual inward bias from the forward samples, which is qualitatively different from the $\lambda^2$ trace error of the zero/wrong-trace parametrizations and from the larger trace errors observed in the GMM setup of Appendix~\ref{app:e4-full}.

\subsection{Baseline parametrizations}
\label{app:baseline-formulas}
All baselines in Table~\ref{tab:baselines} share the MLP backbone $h_\theta$; they differ only in how the score is composed at the boundary. We collect the explicit formulas here for the box $G=[-1,1]^d$ under $A=I$ and reflected OU drift $b(t,x)=-\lambda(t)x$. The anisotropic case in Section~\ref{sec:exp-trace-law} replaces $s_i$ by $s=A^{-1}v$ with $v$ defined analogously per coordinate. In the table, $\beta_I=n^\top b$ denotes the trace target computed as if $A=I$; it coincides with the conormal target $\beta_A=n^\top\beff$ exactly when $A=I$, and the wrong-normal-only baseline enforces $\beta_I$ in place of $\beta_A$.

\paragraph{Main-table reporting convention.} In Table~\ref{tab:fishman-delta}, the signal-to-noise ratio (SNR) is the across-seed gap mean divided by its standard deviation, with $\mathrm{SNR}\ge3$ treated as a clear separation. The $A=\diag(1,4)$ cell uses 5 paired seeds and the rotated cell 8 paired seeds; the two $A=I$ controls use the archived mean\,$\pm$\,standard-deviation summaries and a Welch-combined denominator. The CTR column is structural and has zero seed variance. Shell$_{0.05}$ denotes $\mathrm{ShellErr}_{A,0.05}$ and tang.\ the tangential boundary error.

\begin{table*}[t]
\centering
\small
\begin{tabular}{lrrrr}
\toprule
setting & $\Delta$\,score & $\Delta$\,shell$_{0.05}$ & $\Delta$\,tang. & CTR \\
\midrule
$A{=}I$, $b{=}0$ & 4.08e--2\,(62) & 0.394\,(83) & 1.23\,($>\!10^3$) & $0\to0$ \\
$A{=}I$, $b{=}{-}3x$ & 0.162\,(122) & 4.31\,(154) & 1.23\,($>\!10^3$) & $9\to0$ \\
$A{=}\diag(1,4)$ & 0.287\,(41) & 3.96\,(109) & 1.96\,($>\!10^3$) & $9\to0$ \\
\midrule
\multicolumn{5}{l}{rotated $\kappa{=}4$, multi-active, $m{=}\infty$ (Experiment 2), paired:}\\
\multicolumn{5}{l}{\quad $\Delta$leakNone $0.507$\,(133);\ \ $\Delta$ProjDisp $9.5$\,(66);\ \ $\Delta$face-mass $0.161$\,(64)}\\
\bottomrule
\end{tabular}
\caption{Flux-compatible versus the Fishman-style zero-boundary class under the shared pipeline. $\Delta=$ Fishman $-$ flux (positive favors flux); parentheses give SNR. CTR reports the closed-form Fishman\,$\to$\,flux value. Reporting details are in Appendix~\ref{app:baseline-formulas}.}
\label{tab:fishman-delta}
\end{table*}

\begin{table*}[t]
\centering
\small
\begin{tabular}{lll}
\toprule
baseline & boundary behavior on $\bdry$ & purpose \\
\midrule
\textbf{ours (flux-compatible)} & exact conormal trace $n^\top As=\beta_A$ & proposed \\
zero-boundary (Fishman-style) & score vector vanishes (conormal trace $0$) & over-restriction \\
normal-zero & zero normal trace; tangent free & isolates the trace value \\
wrong-sign & trace inverted (opposite sign) & sign sanity \\
wrong-normal-only & $n^\top s=\beta_I$ as if $A=I$ & conormal ablation (\S\ref{sec:exp-trace-law}) \\
exact-quadrature oracle & explicit boundary integral & oracle in analytic 2D \\
local-time style & stochastic boundary correction & estimator comparison \\
unconstrained & no boundary handling & strawman \\
exact-score oracle & $\nabla\log\rho_t$ in closed form & analytic sanity \\
\bottomrule
\end{tabular}
\caption{Baselines compared in this work. All share the backbone $h_\theta$ and differ only in the boundary composition of the score.}
\label{tab:baselines}
\end{table*}

\begin{itemize}
\item \textbf{ours:} $s_{\theta,i}(t,x)=-\lambda(t)x_i+(1-x_i^2)h_{\theta,i}(t,x)$ (Proposition~\ref{prop:hyperrect}).
\item \textbf{zero-boundary (Fishman-style, 2D):} $s_{\theta,i}(t,x)=(1-x^2)(1-y^2)h_{\theta,i}(t,x)$; the scalar product factor vanishes on every face.
\item \textbf{normal-zero:} $s_{\theta,i}(t,x)=(1-x_i^2)h_{\theta,i}(t,x)$; same per-coordinate factor as ours, no drift term.
\item \textbf{wrong-sign:} $s_{\theta,i}(t,x)=+\lambda(t)x_i+(1-x_i^2)h_{\theta,i}(t,x)$; inverts the conormal trace at the boundary.
\item \textbf{wrong-normal-only} (anisotropic case): $s_{\theta,i}(t,x)=b_i(t,x)+(1-x_i^2)h_{\theta,i}(t,x)$; enforces $n^\top s=n^\top b$ as if $A=I$, which differs from the conormal target $n^\top As=n^\top b$ on faces where $a_i\ne 1$.
\item \textbf{unconstrained + boundary estimator:} $s_\theta=h_\theta$ and an additive local-time-style boundary term computed from one-step reflected-Euler overshoots.
\item \textbf{unconstrained + exact boundary quadrature:} $s_\theta=h_\theta$ plus the exact analytic boundary integral, available only when $\rho_t$ is closed-form (used in the analytic-2D estimator comparison of Appendix~\ref{app:e3-full}).
\item \textbf{exact-score oracle:} $s=\nabla\log\rho_t$ in closed form; used for analytic sampling diagnostics only.
\end{itemize}

\subsection{Tangent-preserving constructions for smooth and signed-distance domains}
\label{app:smooth}
The hyperrectangle construction of Section~\ref{sec:hyperrect-construction} preserves the \emph{tangential} score components while pinning only the \emph{normal} (conormal) component at the boundary. A naive scalar factorization by the defining function $\phi$ of a smooth domain would pin the entire flux-score vector at the boundary, which is precisely the over-restriction we criticize in zero-boundary parametrizations. The correct smooth-domain analogue must therefore decompose the flux-score vector into a prescribed normal component and a free tangential component.

Let $G=\{x\in\R^d:\phi(x)<0\}$ be a bounded domain with smooth boundary $\bdry=\{\phi=0\}$ and outward unit normal $n(y)$ on $\bdry$. In a tubular neighborhood of $\bdry$, let $\pi(x)$ denote the nearest-boundary projection. Let $P_T(y)=I-n(y)n(y)^\top$ be the tangential projector at $y\in\bdry$, and let $\beta_A(t,y)=n(y)^\top \beff(t,y)$ as in the main text. Define
\[
\boxed{\;
\begin{aligned}
v_\theta(t,x)&=\beta_A\!\bigl(t,\pi(x)\bigr)\,n\!\bigl(\pi(x)\bigr)\\
&\quad+P_T\!\bigl(\pi(x)\bigr)\,f_\theta(t,x)+\phi(x)\,g_\theta(t,x),\\
s_\theta&=A^{-1}v_\theta,
\end{aligned}
\;}
\]
where $f_\theta,g_\theta:[0,T]\times\bar G\to\R^d$ are smooth neural fields. On $\bdry$ we have $\phi=0$ and $\pi(x)=x$, so
\[
n^\top v_\theta\big|_{\bdry}
=\beta_A\underbrace{n^\top n}_{=1}
+\underbrace{n^\top P_T}_{=0}f_\theta
+\underbrace{\phi\, n^\top g_\theta}_{=0}
=\beta_A,
\]
which gives $n^\top A_t s_\theta=\beta_A$ on $\bdry$ by construction. The tangential component $P_T f_\theta$ remains free, mirroring the role of the per-axis factor $(1-x_i^2)h_i$ in the box construction.

\paragraph{Signed-distance form (polygonal and disconnected domains).} The construction evaluated in Experiment~4 (Section~\ref{sec:exp-sdf}) is the signed-distance instantiation: with $r(x)$ the signed distance to $\bdry$ (negative inside), $\pi(x)$ the nearest boundary point, $n=n(\pi(x))$ the outward normal there, $P_T=I-nn^\top$ the tangent projector, and $\eta$ a depth gate with $\eta(0)=0$ and $\eta(r)>0$ in the interior,
\[
\begin{aligned}
v_\theta&=\beff+\bigl[P_T(\pi(x))+\eta(r(x))\,n\,n^\top\bigr]h_\theta,\\
s_\theta&=A^{-1}v_\theta.
\end{aligned}
\]
}
On $\bdry$, $\eta=0$ and $P_T h_\theta\perp n$, so $n^\top(v_\theta-\beff)=0$: the conormal trace is pinned while the tangential component of $h_\theta$ is left free. The signed distance is non-smooth at the medial axis and the nearest boundary point can be non-unique on nonconvex or disconnected domains, so this form carries no global no-misspecification theorem; we evaluate it empirically on a polygonal, disconnected domain in Section~\ref{sec:exp-sdf}.

This is a construction sketch. We do not prove a global no-misspecification theorem for smooth domains in this work, and we do not evaluate smooth-domain experiments. The exact no-misspecification theorem (Proposition~\ref{prop:nomisspec}, Appendix~\ref{app:proof-prop4}) is restricted to hyperrectangles. A complete smooth-domain treatment would require globalizing the local tangent decomposition through a partition of unity, controlling the regularity of $\pi$, $n$, and $P_T$ over the tubular neighborhood, and verifying the resulting parametrization empirically; these are natural extensions outside the scope of the present validation.

\subsection{Simplex implementation: construction, representability, and baselines}
\label{app:simplex-impl}

On the probability simplex $\Delta^{K-1}=\{x\in\R^K:x_i\ge 0,\,\mathbf 1^\top x=1\}$ with constant isotropic tangent diffusion $A=aP$ ($P=I-K^{-1}\mathbf 1\mathbf 1^\top$) and reflected OU drift $b(x)=-\lambda(x-c)$, $c=K^{-1}\mathbf 1$, the face $F_i=\{x_i=0\}$ carries an intrinsic outward unit normal $\nu_i=-Pe_i/\|Pe_i\|$ inside the affine hyperplane $\mathbf 1^\top x=1$. The no-flux conormal trace condition on $F_i$ reduces to $v_i=\beff_i(x)\bigl|_{x_i=0}=\lambda/K$ (since $b$ is already tangent and $\nabla\cdot A=0$ for constant $A$). We parametrize the flux-score as
\[
\boxed{\;
\begin{aligned}
v_\theta(t,x)&=b(x)+M_\Delta(x)\,h_\theta(t,x),\\
M_\Delta(x)&=\diag(x)-xx^\top,\\
s_\theta&=A^{-1}v_\theta=v_\theta/a.
\end{aligned}
\;}
\]
The mobility envelope $M_\Delta(x)$ is the standard simplex tangent-vanishing factor and is part of the \emph{parametrization}, not the forward diffusion tensor: $M_\Delta(x)e_i=0$ when $x_i=0$, so on $F_i$ the contribution of $h_\theta$ vanishes and $v_i$ reduces exactly to $b_i$, pinning the conormal trace at $\lambda/K$ by construction.

\paragraph{Representability under $M_\Delta$.} In the relative interior of the simplex, $M_\Delta(x)=\diag(x)-xx^\top$ has rank $K-1$ with kernel spanned by $\mathbf 1$, and its column span equals the tangent space $T_x\Delta^{K-1}=\{u\in\R^K:\mathbf 1^\top u=0\}$; the parametrization $v=b+M_\Delta h$ can therefore represent any element of the interior tangent score class (combined with $M_\Delta^{1/2}$-rescaling of $h$ the interior flux-score map is bijective onto the tangent class). On the face $F_i=\{x_i=0\}$, the $i$-th row of $M_\Delta(x)$ vanishes identically and the remaining $(K-1)\times(K-1)$ block restricts to $M_\Delta$ on the smaller simplex $\Delta^{K-2}$ supported on the other coordinates, which is again rank-$(K-2)$ on its own interior and spans the tangent space of $F_i$; the parametrization restricts on $F_i$ to the same construction one face dimension lower. The face restriction therefore exactly matches the conormal constraint without expressivity loss in the tangential directions, and the only function of $M_\Delta$ at the parametrization level is to algebraically eliminate $h_\theta$ from the \emph{normal} component of $v$ on each face. Vertices ($x_i=1$ for some $i$) sit on a measure-zero stratum on which the construction degenerates; we work on faces in the surface-measure sense and ignore lower-dimensional strata, mirroring the box edges-and-corners treatment of Section~\ref{sec:setup}.

\paragraph{Simplex baselines.} The four structured simplex baselines used in Appendix~\ref{app:e5b-full} are obtained by varying this prescription: flux-compatible $v=b+M_\Delta h$; zero-flux / normal-zero $v=M_\Delta h$; full-boundary-zero $v=(\prod_i x_i)Ph$; wrong-sign $v=-b+M_\Delta h$. A fifth baseline removes the envelope altogether (unconstrained-tangent $v=Ph$). Their predicted face-average CTRs collapse to the closed-form family stated in Appendix~\ref{app:ctr-family} (and verified in Appendix~\ref{app:e5b-full} with zero seed variance).

\subsection{Reproducibility and discretization}
All controlled validation experiments, expanded ablations, and appendix stress tests are complete with seeds, commands, and configurations recorded. For time-conditioned reflected OU experiments, the reverse stopping time $t_{\min}$ and step size must remain below the data mode scale: we use $t_{\min}=0.005$ and $dt=0.005$ for the reported 2D toy experiments. Flower ceiling diagnostics confirm that the stopping time is not the bottleneck on that dataset; the bottleneck is score-matching saturation on the near-1D manifold.

\subsection{Simplex \(K=3\) structural trace identity}
\label{app:e5a}

5 baselines $\times$ 5 seeds, no training, $\lambda=3$, $a=0.5$. Coordinate CTR is the face-average squared residual of $v_i$ on $F_i$, and unit-normal CTR is larger by the fixed factor $K/(K-1)=3/2$. Standard deviation is exactly zero for every structural baseline.

\begin{table}[t]
\centering
\small
\begin{tabular}{lrr}
\toprule
baseline & coord CTR & unit-normal CTR \\
\midrule
flux-compatible & \textbf{0} & \textbf{0} \\
zero-flux & 1.000 & 1.500 \\
full-boundary-zero & 1.000 & 1.500 \\
wrong-sign & 4.000 & 6.000 \\
unconstrained-tangent & 1.006 $\pm$ 0.004 & 1.509 $\pm$ 0.006 \\
\bottomrule
\end{tabular}
\caption{E5a $K=3$ structural trace identity, 5 seeds. For structural baselines predicted and measured CTR coincide exactly (zero seed variance); a single value is shown. Closed-form predictions: $\lambda^2/K^2=1$ (zero-flux, full-boundary-zero), $4\lambda^2/K^2=4$ (wrong-sign), coord-CTR convention; unit-normal CTR is a factor $K/(K-1)=3/2$ larger (Appendix~\ref{app:ctr-family}). Unconstrained-tangent has no prescribed trace and reflects the random initialization of $h_\theta$.}
\label{tab:e5a}
\end{table}

\paragraph{Generation diagnostic (E5a).} For comparison with the trained-model setting, Table~\ref{tab:e5a-gen} reports score error, shell error, MMD$^2$ and leakNone for the same five baselines after $10{,}000$ ISM training steps at $K=3$ (5 seeds). Score error here is the $A$-weighted error $\E[(s_\theta-s^\star)^\top A(s_\theta-s^\star)]$ on stationary reference samples, averaged over the evaluation times (the target is the stationary distribution, so the truth is time-independent), computed at the flux level against the exact stationary flux $v^\star=As^\star=b$, the barycenter-directed drift (equivalently $s^\star=b/a$ on the tangent space). The score-/shell-error ordering across baselines is not monotone in trace correctness in this regime (e.g., wrong-sign attains the lowest score error $3.98$; flux-compatible is $15.91$); the ordering is dominated by differential susceptibility to the interior divergence-overfitting pathology (Appendix~\ref{app:simplex-pathology}) and is not informative about boundary handling. The clean separation in this experiment is on leakNone (flux 0.036 vs zero-flux 0.527, wrong-sign 0.615); we do not interpret the score / shell error ranking as a quality ordering.

\begin{table*}[t]
\centering
\small
\begin{tabular}{lrrrrr}
\toprule
baseline & score err (A) & shell$_{0.02}$ & shell$_{0.05}$ & MMD$^2$ & leakNone \\
\midrule
flux-compatible & 15.91 $\pm$ 2.7 & 24.15 $\pm$ 10.1 & 53.90 $\pm$ 13.8 & 0.0203 $\pm$ 0.0046 & \textbf{0.036 $\pm$ 0.026} \\
zero-flux & 6.46 $\pm$ 1.3 & 8.99 $\pm$ 2.9 & 15.17 $\pm$ 5.7 & 0.0264 $\pm$ 0.0048 & 0.527 $\pm$ 0.241 \\
full-boundary-zero & 8.18 $\pm$ 1.0 & 11.26 $\pm$ 0.9 & 19.06 $\pm$ 3.2 & 0.0646 $\pm$ 0.0053 & 0.058 $\pm$ 0.057 \\
wrong-sign & 3.98 $\pm$ 0.17 & 12.60 $\pm$ 0.5 & 10.56 $\pm$ 0.2 & 0.0384 $\pm$ 0.0029 & 0.615 $\pm$ 0.220 \\
unconstrained-tangent & 12.23 $\pm$ 1.9 & 77.45 $\pm$ 13.7 & 51.96 $\pm$ 8.8 & 0.00387 $\pm$ 0.0012 & 0.0001 $\pm$ 0.00022 \\
\bottomrule
\end{tabular}
\caption{E5a generation diagnostic at $K=3$, $\lambda=3$, $a=0.5$, 10k training steps, mean $\pm$ std over 5 seeds. The score / shell error columns are reported for completeness; they are not monotone in trace correctness in this regime, and the operative diagnostic is leakNone.}
\label{tab:e5a-gen}
\end{table*}

\subsection{Simplex moderate-K full tables (K=16, K=32 synthetic, K=32 HITChip)}
\label{app:e5b-full}

All three datasets share the training protocol: $20{,}000$ ISM steps, Hutchinson divergence ($R=16$), Adam $\mathrm{lr}=10^{-4}$, gradient clipping max-norm 1, batch 256, MLP width 128 depth 4 SiLU. Boundary mass profiles of the datasets are summarized below.

\begin{table*}[t]
\centering
\small
\begin{tabular}{lrrr}
\toprule
& K=16 synth ($\tau=1.5$, $\mu$-scale 2.0) & K=32 synth ($\tau=2.5$, $\mu$-scale 3.0) & K=32 HITChip \\
\midrule
$P(\min_i x_i<10^{-3})$ & 0.274 & 0.490 & 0.652 \\
$P(\min_i x_i<10^{-2})$ & 0.979 & 1.000 & 1.000 \\
\bottomrule
\end{tabular}
\caption{Boundary mass profiles for the three moderate-$K$ simplex datasets.}
\label{tab:e5b-boundary-mass}
\end{table*}

Table~\ref{tab:e5b} retains only the face-CTR diagnostic, which is independent of reverse-sampler initialization. Moderate-$K$ generation metrics from the earlier data-start protocol are not reported; we do not rerun this diagnostic suite because it is outside the main evidence chain. The full-boundary-zero loss\_init and loss\_final are exactly zero in every cell of K=16, K=32, and HITChip.

\begin{table*}[t]
\centering
\small
\begin{tabular}{lrrrrr}
\toprule
dataset & flux & zero-flux & full-zero & wrong-sign & unconstr. \\
\midrule
K=16 synth & \textbf{0} & 0.0352 & 0.0352 & 0.141 & $196\pm122$ \\
K=32 synth & \textbf{0} & 8.79e--3 & 8.79e--3 & 3.52e--2 & $170\pm58$ \\
K=32 HITChip & \textbf{0} & 8.79e--3 & 8.79e--3 & 3.52e--2 & $187\pm60$ \\
\bottomrule
\end{tabular}
\caption{E5b moderate-$K$ face-average CTR. Structural predictions and measurements coincide to floating-point precision with zero seed variance over 5 seeds; unconstrained-tangent has no prescribed trace and is reported as mean $\pm$ std.}
\label{tab:e5b}
\end{table*}

\paragraph{Mechanism (full-boundary-zero only).} This degeneracy is specific to the full-boundary-zero baseline, whose face-vanishing factor is the product envelope $\prod_i x_i$ (not the flux-compatible mobility envelope $M_\Delta$, which does not underflow). At K=16 the product envelope $\prod_i x_i$ is bounded above by $K^{-K}\approx 5.4\times 10^{-20}$ at the simplex barycenter (well above the float32 normal floor $\approx 1.2\times 10^{-38}$, so this is not classical IEEE underflow) and is many orders smaller on the boundary-supported training data; the multiplicative attenuation places the trainable part of $v$ below the optimizer's effective precision relative to the unaffected $b$ term, so gradients vanish and the network never trains. At K=32 the barycenter value $K^{-K}\approx 7\times 10^{-49}$ falls below the float32 normal floor outright. Either way the operative empirical signature is loss\_init = loss\_final = 0 in every full-boundary-zero cell; we report this as a numerical degeneracy of that baseline rather than as a competitive generation result.

\paragraph{HITChip dataset.} The Lahti et al.\ HITChip Atlas \citep{lahti2014tipping} contains 1172 stool samples from 1006 western adults, profiled at 130 genus-like taxa via the HITChip microarray (Dryad DOI 10.5061/dryad.pk75d, file \texttt{HITChip.tab}). We aggregate to the top-31 genus-like taxa by mean row-normalized abundance, lump the remaining 99 taxa into an ``Other'' coordinate (yielding K=32), and apply a uniform relative pseudocount of $10^{-6}$ before renormalizing. HITChip is a real-data structural diagnostic: flux retains $\mathrm{CTR}=0$ while unconstrained-tangent measures $187\pm60$. It is not a generative-quality benchmark.

\subsection{E6 structural and training equality tables (\(d=2\) and \(d=4\))}
\label{app:e6-structural}
\label{app:e6-training}

Box $[-1,1]^d$ with diagonal state-dependent diffusion $A_\kappa(x)=\diag(1+\kappa x_i)$, drift $b(x)=-\lambda x$, $\lambda=3$. Baselines:
\begin{itemize}
\item \textbf{correct}: $v_i = \beff_i + (1-x_i^2)h_i = (-\lambda x_i - \kappa) + (1-x_i^2)h_i$ on faces;
\item \textbf{omit-correction}: $v_i = b_i + (1-x_i^2)h_i$ (drops the $-\nabla\!\cdot\!A=-\kappa$ correction);
\item \textbf{normal-zero}: $v_i = (1-x_i^2)h_i$ (drops both $b_i$ and $-\kappa$);
\item \textbf{wrong-sign}: $v_i = -\beff_i + (1-x_i^2)h_i$ (negates the effective trace).
\end{itemize}

\paragraph{E6 structural (no training).} 5 random initializations per $(\kappa,\text{baseline})$ cell, 7 $\kappa$ values, 4 baselines, $d\in\{2,4\}$. Measured face-average CTR equals predicted value to floating-point precision; per-seed standard deviation is exactly zero. The predicted values are the closed-form family from Appendix~\ref{app:ctr-family} and the body of Table~\ref{tab:e6-training} (which repeats them verbatim).

\begin{table*}[t]
\centering
\small
\begin{tabular}{lrrrr}
\toprule
$\kappa$ & correct & omit (= $\kappa^2$) & normal-zero (= $\lambda^2+\kappa^2$) & wrong-sign (= $4(\lambda^2+\kappa^2)$) \\
\midrule
$-0.75$ & 0 & 0.5625 & 9.5625 & 38.25 \\
$-0.50$ & 0 & 0.2500 & 9.2500 & 37.00 \\
$-0.25$ & 0 & 0.0625 & 9.0625 & 36.25 \\
$0$ & 0 & 0 & 9.00 & 36.00 \\
$+0.25$ & 0 & 0.0625 & 9.0625 & 36.25 \\
$+0.50$ & 0 & 0.2500 & 9.2500 & 37.00 \\
$+0.75$ & 0 & 0.5625 & 9.5625 & 38.25 \\
\bottomrule
\end{tabular}
\caption{E6: measured post-training face-average CTR equals the closed-form predicted value to floating-point precision in every cell at $d=2$, with zero seed variance, and is identical before training. The wrong-sign baseline minimizes the mis-specified ISM loss from $\approx +8$ at initialization down to $+0$--$+3$ over 5000 steps (its loss is bounded below) without changing its face CTR. Numbers are mean (std) over 5 seeds; std is exactly zero in every cell. The $d=4$ table is identical (face-CTR is dimension-independent for these baselines under the chosen parametrization).}
\label{tab:e6-training}
\end{table*}

\paragraph{E6 training.} 5000 ISM steps per cell, Adam $\mathrm{lr}=3\times 10^{-4}$, gradient clipping max-norm 1, Hutchinson divergence $R=4$, batch 256, MLP width 64 depth 3 SiLU, reflected-OU forward bank with state-dependent diffusion: $dX=-\lambda X\,dt+\sqrt{2A_\kappa(X)}\,dW-n(X)\,dL$ (the paper Section~\ref{sec:setup} / Section~\ref{sec:exp-trace-law} It\^o-convention SDE; $b=-\lambda x$ is itself the It\^o drift in the FP flux $J=b\rho-A\nabla\rho-\rho\nabla\cdot A$). All $140$ cells at $d=2$ and all $140$ cells at $d=4$ yield post-training measured CTR equal to predicted to floating-point precision, with zero seed variance. No NaN losses were observed.

Representative loss trajectories ($\kappa=-0.75$, $d=2$, per-seed loss\_init $\to$ loss\_final ranges):
\begin{itemize}
\item correct: $-3.83$ to $-3.19 \to -3.97$ to $-3.43$ (converges to a divergence-bounded minimum)
\item omit: $-3.17$ to $-2.47 \to -4.00$ to $-3.45$
\item normal-zero: $+0.00$ to $+0.03 \to -2.02$ to $-1.04$ (decreases substantially)
\item wrong-sign: $+7.98$ to $+8.39 \to +0.03$ to $+2.05$ (large decrease; loss bounded below in this case)
\end{itemize}

For $d=4$ the same $\kappa=-0.75$ row has wider ranges due to the additional faces: correct $-7.68$ to $-6.47 \to -7.70$ to $-7.18$; omit $-6.31$ to $-4.94 \to -7.50$ to $-6.94$; normal-zero $-0.02$ to $+0.02 \to -3.78$ to $-2.55$; wrong-sign $+15.93$ to $+17.05 \to +1.71$ to $+4.70$. The post-training measured CTR equality is preserved verbatim at both dimensions. Full per-seed data is in the supplementary material.

\begin{figure*}[t]
\centering
\includegraphics[width=0.65\textwidth]{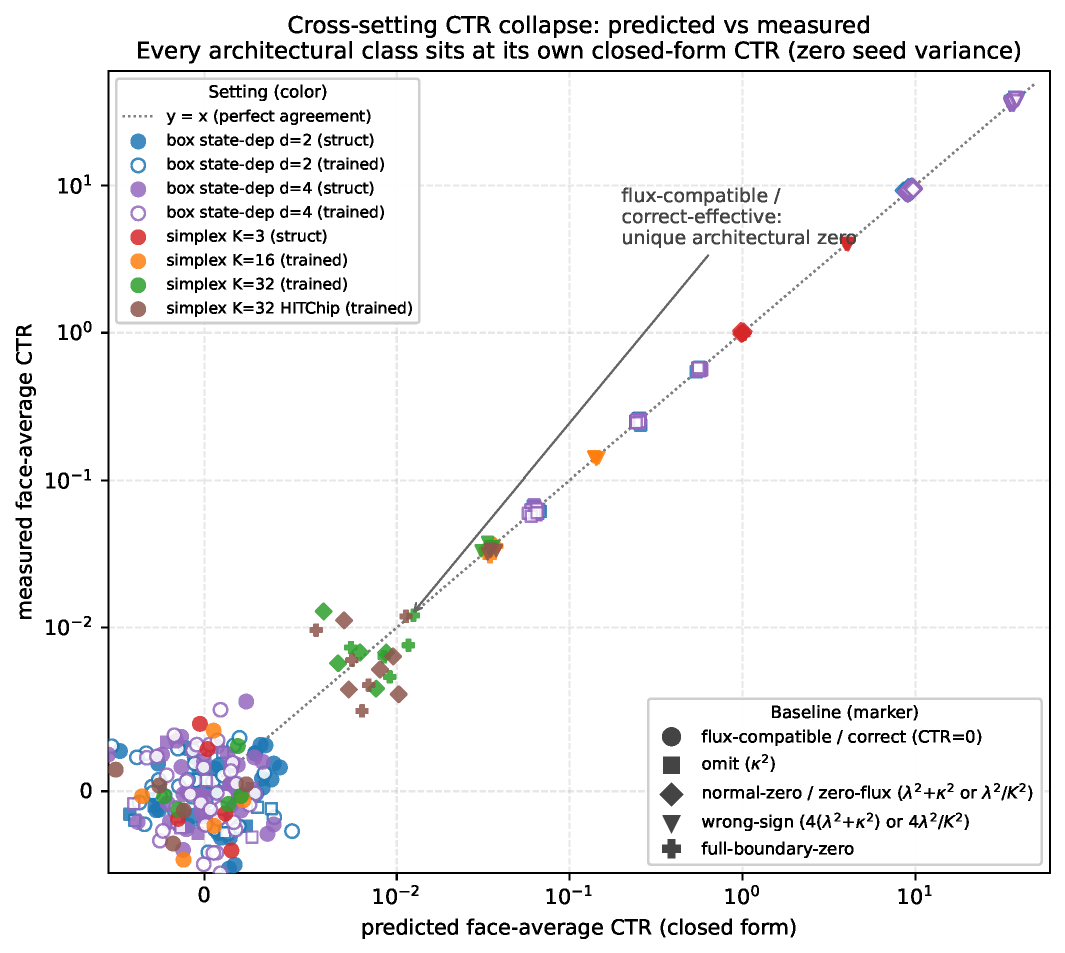}
\caption{Every architectural class lies on the diagonal $\mathrm{measured}=\mathrm{predicted}$ in all 640 cells, before and after training. Cross-setting CTR collapse: predicted vs measured face-average CTR across the box state-dependent grid (E6 structural and post-training, $d\in\{2,4\}$, 7 $\kappa$ values, 4 baselines, 5 seeds), the simplex $K=3$ structural trace identity (E5a, 4 baselines, 5 seeds), and the simplex moderate-$K$ face-CTR column at $K=16$, $K=32$ synthetic, and $K=32$ HITChip (E5b, 4 structural baselines, 5 seeds). The structural baselines (filled markers) have exactly zero seed variance (jitter added only to separate coincident cells). The unique lower-left zero is the flux-compatible class across all settings; every other class is fixed at its own non-zero predicted value ($\kappa^2$, $\lambda^2+\kappa^2$, $4(\lambda^2+\kappa^2)$ on the state-dependent box; $\lambda^2/K^2$, $4\lambda^2/K^2$ on the simplex), so the CTR identity distinguishes classes by closed-form arithmetic. Open markers (state-dependent $d=2$/$d=4$ trained) overlap the filled structural markers cell-by-cell --- the numerical-faithfulness statement that aggressive ISM optimization does not break the by-construction CTR identity.}
\label{fig:ctr-collapse}
\end{figure*}

\subsection{Moderate-K simplex optimization pathology}
\label{app:simplex-pathology}

The interior ISM objective on bounded domains decomposes as $s^\top As+2\nabla\cdot(As)$; both terms are evaluated at training samples, and the divergence term is high-variance and unbounded below pointwise. On the moderate-dimensional simplex the tangent-Hutchinson estimate lets the optimizer drive the empirical loss downward by inflating $\|v\|$ locally at sample points. Reflected DSM targets can be constructed from constrained transition kernels or approximations \citep{lou2023reflecteddiffusionmodels}; we use ISM because it isolates the conormal boundary term across our tensors and geometries without a domain-specific conditional-score construction. Moderate-$K$ generation metrics from the earlier data-start reverse protocol are omitted rather than reinterpreted.

\paragraph{Instrumented loss trajectory.} Over 5000 steps, the quadratic term grows by $\approx230\times$ (from 2 to 460), the divergence term falls from $-186$ to $-1263$, $\|v_\theta\|_2$ grows by $\approx13\times$, and the raw parameter-gradient norm by $\approx420\times$. These training-side quantities do not depend on reverse initialization and support the narrower diagnosis that the unregularized empirical ISM objective develops divergence wells. We make no moderate-$K$ generation-quality claim.

\subsection{Partial-repair stress spectrum (box, \(d=2\), \(d_{\mathrm{bdry}}=0.05\), \(\lambda=3\))}
\label{app:partial-repair}

This stress test varies how often the reverse sampler enforces the reflective fold, decoupling the score from the sampler's external repair mechanism. Setup: same boundary-hugging GMM data and four box baselines (correct flux-compatible, normal-zero, Fishman, wrong-sign) as Table~\ref{tab:a3-full}, 5 seeds, 6000 ISM training steps. The reverse sampler is run with a \emph{periodic-projection} boundary mode that mirror-reflects to $[-1,1]$ every $m$ reverse steps; between projections the sampler is soft-clamped to $[-8,8]$ for numerical safety. $m=1$ recovers the standard fold protocol, $m=\infty$ recovers the existing no-fold protocol of Table~\ref{tab:a3-full}. We additionally sweep the number of reverse steps $N\in\{100,250,500,1000\}$ at $m=1$, the more practically-relevant axis of repair-weakening.

\paragraph{Sanity gate.} The retrained-models reproduction of the published Table~\ref{tab:a3-full} endpoint is, at $m=\infty$ and $d_{\mathrm{bdry}}=0.05,\,\lambda=3$: ours $0.221$, normal-zero $0.988$, Fishman $0.971$, wrong-sign $0.993$ (5-seed mean), all within $\pm 0.05$ of the published Table~\ref{tab:a3-full} entries. At $m=1$, all four baselines have leak $0$. Both endpoints reproduce; this is the sampler validation.

\paragraph{m-axis (Table~\ref{tab:partial-repair-m}).} Hard-reflected MMD$^2$ measured on the folded sample (i.e.\ after a final post-sampling fold to $[-1,1]$). Mean $\pm$ std over 5 seeds.

\begin{table*}[t]
\centering
\small
\begin{tabular}{lrrrrrr}
\toprule
baseline & m=1 & m=2 & m=3 & m=5 & m=8 & m=$\infty$ \\
\midrule
ours & 0.0248 $\pm$ 8.6e--3 & 0.0264 $\pm$ 9.6e--3 & 0.0250 $\pm$ 1.1e--2 & 0.0252 $\pm$ 1.0e--2 & 0.0229 $\pm$ 9.1e--3 & 0.0196 $\pm$ 7.4e--3 \\
normal-zero & 0.0224 $\pm$ 4.0e--3 & 0.0240 $\pm$ 4.6e--3 & 0.0305 $\pm$ 4.5e--3 & 0.0400 $\pm$ 8.4e--3 & 0.0594 $\pm$ 1.1e--2 & 0.677 $\pm$ 0.33 \\
Fishman & 0.0279 $\pm$ 1.5e--3 & 0.0336 $\pm$ 5.4e--3 & 0.0426 $\pm$ 2.1e--3 & 0.0519 $\pm$ 6.1e--3 & 0.0856 $\pm$ 5.8e--3 & 1.330 $\pm$ 0.023 \\
wrong-sign & 0.0404 $\pm$ 1.4e--2 & 0.0507 $\pm$ 1.3e--2 & 0.0744 $\pm$ 1.2e--2 & 0.104 $\pm$ 0.018 & 0.142 $\pm$ 0.025 & 0.897 $\pm$ 0.29 \\
\bottomrule
\end{tabular}
\caption{Partial-repair m-axis: hard-reflected MMD$^2$ versus periodic-projection period $m$. At $m=1$ (standard practice), ours and Fishman are statistically tied; separation begins to emerge at $m\ge 5$ and is statistically clear (SNR $\ge 5$) only at $m\ge 8$. Each cell is mean $\pm$ std over 5 seeds.}
\label{tab:partial-repair-m}
\end{table*}

\paragraph{N-axis (Table~\ref{tab:partial-repair-N}).} The more practical axis: reduce the number of reverse steps at $m=1$ (the standard fold-every-step protocol). Here ours and Fishman are statistically tied across the entire practical range.

\begin{table*}[t]
\centering
\small
\begin{tabular}{lrrrr}
\toprule
baseline & N=1000 & N=500 & N=250 & N=100 \\
\midrule
ours & 0.0227 $\pm$ 1.1e--2 & 0.0234 $\pm$ 9.7e--3 & 0.0300 $\pm$ 8.3e--3 & 0.0620 $\pm$ 7.1e--3 \\
normal-zero & 0.0162 $\pm$ 3.7e--3 & 0.0205 $\pm$ 5.5e--3 & 0.0297 $\pm$ 6.2e--3 & 0.0634 $\pm$ 7.1e--3 \\
Fishman & 0.0237 $\pm$ 3.6e--3 & 0.0279 $\pm$ 3.7e--3 & 0.0344 $\pm$ 4.9e--3 & 0.0659 $\pm$ 4.2e--3 \\
wrong-sign & 0.0355 $\pm$ 1.7e--2 & 0.0380 $\pm$ 1.6e--2 & 0.0459 $\pm$ 1.1e--2 & 0.0782 $\pm$ 1.5e--2 \\
\bottomrule
\end{tabular}
\caption{Partial-repair N-axis (m=1): hard-reflected MMD$^2$ versus reverse-step count $N$. Within the practical range $N\in\{100,\,250,\,500,\,1000\}$, ours and Fishman are statistically tied.}
\label{tab:partial-repair-N}
\end{table*}

\paragraph{Honest reading.} At the standard protocol $m=1$, ours and Fishman have a hard-reflected MMD$^2$ gap of $+3.1\times 10^{-3}$ with combined standard deviation $\approx 9\times 10^{-3}$ --- tied within seed noise (signal-to-noise ratio $0.4$). Across the practical N-axis, the gap is $+0.001$ to $+0.004$ with signal-to-noise ratio $\le 0.5$ at every N. The structural separation reported in Table~\ref{tab:a3-full} and Figure~\ref{fig:fold-vs-none} is recovered only at $m\ge 8$ (signal-to-noise ratio $5.8$ in Table~\ref{tab:partial-repair-m}) or at $m=\infty$. This is the reflection-masking thesis observed continuously rather than at the endpoint: as the repair weakens, the wrong conormal trace starts to manifest in distributional metrics. It is not a hard-reflected generation-quality win in standard usage, and we treat it as supporting evidence for the masking discussion in Section~\ref{sec:disc-principle}, not as a new headline. This single-tensor box stress is the precursor that Experiment 2 (Appendix~\ref{app:phase-diagram-full}) generalizes to a full $(\kappa,m)$ grid with rotated $A$ and operational metrics.

\subsection{Experiment 2 full phase diagram (P1)}
\label{app:phase-diagram-full}

\paragraph{Setup.} Box $[-1,1]^2$, both faces active (corner-hugging GMM, mode-distance/std $\approx1.5$, near-face mass $\approx0.46$ --- resolvable, not a spike), $\delta=0.10$, $\lambda=3$. Constant SPD $A=Q\,\diag(1,\kappa)\,Q^\top$, $\kappa\in\{1,4,16\}$, $Q$ axis-aligned or a random rotation. Six architecture-matched baselines (flux-compatible, wrong-normal-only, zero-boundary, normal-zero, wrong-sign, unconstrained), exact ISM divergence, shared MLP/optimizer, gradient-norm clip $50$ --- a non-capping stabilizer that holds the high-$\kappa$ divergence-overfitting baselines (ISM loss $\approx-1700$ at $\kappa=16$) finite without touching any score parametrization. Partial-repair sampler folds every $m\in\{1,2,4,8,\infty\}$ steps; 3 seeds over the grid, 8 at the $\kappa=4$ rotated corner. Primary metrics are in-box face-mass error and projection displacement. The verdict rule and the directional prediction reported below were fixed before the runs (recorded in the experiment driver and \texttt{FINDINGS\_phase\_diagram.md}): a positive result requires the flux score to beat the strongest non-flux baseline with SNR $\ge3$ on the primary metrics, with the advantage increasing as $\kappa$ grows or repair weakens.

\paragraph{Result.} At rotated $A$, $m=\infty$ (gain $=$ best-non-flux error $-$ flux error, SNR$\,\ge3$ = clear): $\kappa=1$ control face $-0.003$ (SNR $-0.65$), proj $-0.515$ (SNR $-15.9$) --- no advantage, since flux $=$ wrong-normal at $A=I$; the per-cell best of five baselines here is the unconstrained baseline, so the negative gain is best-of-baselines selection at a null cell, not a deficiency of the construction; $\kappa=4$ face $+0.097$ (SNR $29.4$), proj $+6.77$ (SNR $3.19$) --- both at the same corner; $\kappa=16$ face $-0.001$ (SNR $-0.10$, saturated), proj $+4.32$ (SNR $9.39$). The $\kappa=16$ projection-displacement gain grows monotonically with weaker repair: $m=1\!:0.00$, $m=2\!:0.01$, $m=4\!:0.03$, $m=8\!:0.22$, $m=\infty\!:4.32$ (leaked fraction at $m=\infty$: flux $0.43$ vs.\ wrong-normal $1.00$). Under full repair the face-mass placement does not favor flux (SNR $-4.7$ at $\kappa=4$, $-2.7$ at $\kappa=16$). The 8-seed $\kappa=4$ supplement tightens proj-disp SNR from $1.62$ (3 seeds; gain already $+6.15$) to $3.19$; the result is positive under the pre-registered rule even without it (face SNR $31$ at $\kappa=4$, proj SNR $9.4$ at $\kappa=16$, at different corners). The pre-registered ``advantage increases with $\kappa$'' is refined to ``increases as repair weakens; non-monotone in $\kappa$, peaking at moderate anisotropy.'' Full per-condition data is in the supplementary material.

\subsection{Experiment 3 full real-design posterior (P2)}
\label{app:real-design-full}

\paragraph{Setup.} Target $\mathcal{N}(\mu,\Sigma_{\mathrm{post}})\mathbf{1}_{\mathrm{box}}$ with $\Sigma_{\mathrm{post}}$ from the diabetes design matrix (real correlation $R=[[1,.68,-.11],[.68,1,-.52],[-.11,-.52,1]]$), all three coordinates active, per-coordinate offsets set so each active face has near-boundary mass $\approx0.5$ (realized $0.27/0.61/0.70$ --- resolvable). $\lambda=3$, Gibbs (TMVG) reference. $A=Q\,\diag(1,\kappa)\,Q^\top$ with $Q$ from the eigenvectors of $\Sigma_{\mathrm{post}}$ (data-aligned direction mismatch $49.4$ at $\kappa=16$) plus an axis-aligned control; $\kappa\in\{1,4,16\}$; the same six baselines, exact ISM, repair $m\in\{1,2,4,8,\infty\}$, 3 seeds. Face-mass error is averaged over the three active faces.

\paragraph{Result.} At rotated $A$, $m=\infty$: $\kappa=1$ control face $-0.001$ (SNR $-0.60$), proj $-0.703$ (SNR $-17.2$) --- as in Experiment 2, flux $=$ wrong-normal at $A=I$ and the per-cell best of five baselines is the unconstrained baseline, so the negative gain is best-of-baselines selection at a null cell, not a deficiency of the construction; $\kappa=4$ face $+0.099$ (SNR $15.1$), proj $+6.27$ (SNR $1.25$, noise-limited by the blow-up baseline variance); $\kappa=16$ face $+0.035$ (SNR $12.1$), proj $+2.11$ (SNR $8.32$) --- both clear at the same corner, no seed supplement. The $\kappa=16$ projection-displacement gain vs.\ repair: $m=1\!:0.00$, $m=2\!:-0.03$, $m=4\!:0.06$, $m=8\!:0.20$, $m=\infty\!:2.11$. Unlike P1, the face-mass effect does not saturate at $\kappa=16$, and full-repair placement is mildly favorable ($m=1$ SNR $+0.7$ at $\kappa=4$, $+5.0$ at $\kappa=16$) --- the opposite sign from P1, which is why Section~\ref{sec:disc-principle} reads placement as condition-dependent. Full data is in the supplementary material.

\subsection{Experiment 4 full SDF polygon (P3)}
\label{app:sdf-full}

\paragraph{Corrected protocol.} The disconnected domain contains a nonconvex L-polygon, a square island, and a narrow rectangle. The target is a truncated GMM with modes near each component's boundary; $A=R_\phi\diag(1,4)R_\phi^\top$ with $\phi=\pi/6$. Five architecture-matched baselines share the network, optimizer, non-capping Jacobian stabilizer $\lambda_{\mathrm{Jac}}=10^{-2}$, and projection schedule $m\in\{1,2,4,8,\infty\}$. Reverse samples start from an independent finite-horizon marginal $q_T$, obtained by evolving fresh target draws through the reflected forward process to $T=1$. We run 3 seeds; complete configurations and per-seed metrics are in the supplementary code.

\begin{table*}[t]
\centering
\small
\begin{tabular}{lrrrr}
\toprule
method & CTR & ProjDisp & band error & component KL \\
\midrule
tangent-preserving & $0$ & $0.224\pm0.013$ & $0.183\pm0.001$ & $0.209\pm0.119$ \\
scalar mask & $0$ & $0.407\pm0.004$ & $0.226\pm0.002$ & $0.013\pm0.001$ \\
zero-boundary & $4.18\pm0.05$ & $2.602\pm0.002$ & $0.300\pm0.000$ & $0.146\pm0.002$ \\
\bottomrule
\end{tabular}
\caption{Corrected-$q_T$ SDF results at $m=\infty$, mean $\pm$ std over 3 seeds (lower is better except CTR, whose target is zero).}
\label{tab:sdf-qt-final}
\end{table*}

\paragraph{Result.} Against the equal-CTR scalar mask, the tangent-preserving construction reduces ProjDisp (SNR $13.60$) and boundary-band error (SNR $18.21$) but not component-mass KL, which is the explicit negative result in Table~\ref{tab:sdf-qt-final}. Its ProjDisp advantage grows monotonically as repair weakens: $0.000,0.006,0.015,0.031,0.183$ for $m=1,2,4,8,\infty$. It also beats zero-boundary on ProjDisp and boundary-band error while satisfying the trace exactly. Thus P3 supports the reflection-masking mechanism and tangent-preserving extension on this disconnected non-axis-aligned domain, but it is a supplemental method-generalization test, not a global no-misspecification or generation-quality claim.

\subsection{Negative controls: single-active-face constrained posteriors (P5)}
\label{app:neg-controls}

Two real-data constrained-posterior controls test whether the Experiment 2--3 effect appears when its preconditions are removed, and isolate two candidate confounds.

\paragraph{(a) Near-delta spike.} The diabetes regression posterior over $(s3,s4,s5)$ truncated to $[0,2]^3$, where the OLS coefficient of the hugged coordinate is $\approx-0.233<0$, so the non-negativity constraint is binding and the truncated posterior is a near-delta spike on the face (reference near-boundary mass $0.999$, $P(\beta_{s3}<0.05)=0.981$). Reflected OU with exact ($d=3$) ISM divergence, fold sampler, Gibbs reference. \emph{Outcome: inconclusive.} The overall posterior is reproduced (sanity energy distance (ED) $0.140$), but all three baselines underfit the spike (near-boundary-mass error $0.648/0.627/0.641$ for flux / wrong-normal / normal-zero against a reference near-boundary mass of $0.999$; i.e., the models place only $\approx35\%$ of the reference near-boundary mass), and the flux-vs-baseline gaps ($\approx0.02$) are an order of magnitude smaller than the shared $\approx0.63$ underfitting. The near-boundary error grows monotonically with anisotropy ($0.37\to0.75$ over $\rho\in\{1,2,4,8\}$), consistent with a continuous reflected diffusion smoothing a near-delta target equally for every boundary-handling choice. This is a target-sharpness limitation, not evidence about boundary correctness; we did not tune.

\paragraph{(b) Resolvable single active face.} Removing the sharpness confound while keeping the real correlation: a controlled truncated Gaussian $\mathcal{N}(\mu,\delta_t^2 R)\mathbf{1}_{\mathrm{box}}$ with the hugged-coordinate mode at distance $\approx\delta_t$ from a single face (mode-distance/std $\approx1$, reference near-face mass $0.40$, $s3$ std $0.079$), $\delta_t=0.10$, $\rho\in\{1,4\}$, 3 seeds, sanity ED $0.043$ (well-trained). \emph{Outcome: tie.} At the primary cell (fold, $\rho=4$) flux is marginally worst on every metric but not separated (near-boundary error $0.119$ vs.\ $0.108/0.105$, SNR $\le1.25$); the structure-free normal-zero baseline is best under the fold, consistent with the fold masking boundary handling. \emph{Correlation ruled out:} rerunning with $R$ replaced by the identity (same seeds) gives a near-identical flux-vs-wrong gap ($+0.0117$ real-$R$ vs.\ $+0.0128$ at $R=I$ on near-boundary error), so the real correlation structure is not the cause; the tie is intrinsic to the single-active-face fold geometry. No spectral-norm cap is used, the sanity and isotropy ($\rho=1$) controls pass, and $R$ vs.\ $I$ is a same-seed paired comparison, so the null is genuine --- no claim that flux is significantly worse, only that there is no advantage here and correlation is not why. Together these controls show the Experiment 2--3 effect requires simultaneously active constraints and imperfect repair, and is absent --- for reasons we can attribute --- when either is removed. Full data is in the supplementary material.

\subsection{Prediction-vs-outcome scorecard}
\label{app:scope}

Table~\ref{tab:scope} collects the status of every regime referenced from Section~\ref{sec:disc-principle}: each row links a regime to the reflection-masking hypothesis it tests (M1--M4) or marks a control / condition-dependent observation (---). ``Positive'' results are scoped to the stated regime; ``null''/``inconclusive'' entries confirm a null prediction or locate the effect.

\begin{table*}[t]
\centering
\small
\begin{tabular}{lp{5.3cm}l}
\toprule
pred. & regime: observation & status \\
\midrule
M1 & full repair ($m{=}1$): ProjDisp gain $\approx0$, face-mass sign condition-dependent (no consistent trace signal) & confirmed \\
M2 & tested $A{=}I$: flux $=$ wrong-normal-only, no distinction between these two classes & control \\
M3 & tested single active face, full repair: flux tied (fold masks it) & null \\
M4 & multi-active, weakened repair: lower ProjDisp and face-mass error & positive (scoped) \\
M4$'$ & SDF polygon $+$ stabilizer: lower ProjDisp and boundary-band error & positive (cond.) \\
--- & full-repair face-mass placement: sign flips across Exp.\ 2--3 ($-4.7\sigma$/$+5.0\sigma$) & cond.-dep. \\
--- & extreme $\kappa$: face-mass gain saturates (over-spread) & non-monotone \\
--- & SDF polygon, no stabilizer: interior divergence overfits & null (orth.) \\
--- & near-delta posterior: all methods underfit the layer & inconclusive \\
--- & component mass, disconnected domain: not trace-set & null \\
\bottomrule
\end{tabular}
\caption{Prediction-vs-outcome scorecard.}
\label{tab:scope}
\end{table*}

\fi
\bibliography{ref}
\end{document}